\documentclass{article}

\usepackage[preprint]{neurips_2024}

\usepackage[utf8]{inputenc} 
\usepackage[T1]{fontenc}    
\usepackage{hyperref}       
\usepackage{url}            
\usepackage{booktabs}       
\usepackage{amsfonts}       
\usepackage{nicefrac}       
\usepackage{microtype}      
\usepackage{xcolor}         
\usepackage{amsmath}
\usepackage{amssymb}
\usepackage{mathtools}
\usepackage{amsthm}
\usepackage{svg}
\newtheorem{theorem}{Theorem}[section]
\newtheorem{corollary}{Corollary}[theorem]
\newtheorem{lemma}[theorem]{Lemma}

\newtheorem{definition}[theorem]{Definition}
\usepackage{algorithm}
\usepackage{algpseudocode}
\usepackage{multirow}
\usepackage{wrapfig}
\usepackage{lipsum}
\usepackage{caption}
\usepackage{appendix}
\usepackage{titletoc}
\newcommand\DoToC{%
  \startcontents
  \printcontents{}{2}{\textbf{Contents}\vskip3pt\hrule\vskip5pt}
  \vskip3pt\hrule\vskip5pt
}

\title{Towards training digitally-tied analog blocks \\
via hybrid gradient computation}

\author{%
  Timothy~Nest\thanks{Montreal Institute of Learning Algorithms (MILA)}$\ \  ^{\blacklozenge}$ \\
  \texttt{timothy.nest@mila.quebec} \\
  \And
  Maxence~Ernoult\thanks{Rain AI}$\ \  ^{\blacklozenge}$  \\
  \texttt{maxence@rain.ai}
}

\begin{document}

\maketitle

\begin{abstract}
\def\thefootnote{$\blacklozenge$}\footnotetext{Equal contribution}\def\thefootnote{\arabic{footnote}}
Power efficiency is plateauing in the standard digital electronics realm such that novel hardware, models, and algorithms are needed to reduce the costs of AI training. The combination of energy-based analog circuits and the Equilibrium Propagation (EP) algorithm constitutes one compelling alternative compute paradigm for gradient-based optimization of neural nets. Existing analog hardware accelerators, however, typically incorporate digital circuitry to sustain auxiliary non-weight-stationary operations, mitigate analog device imperfections, and leverage existing digital accelerators.This heterogeneous hardware approach calls for a new theoretical model building block. In this work, we introduce \emph{Feedforward-tied Energy-based Models} (ff-EBMs), a hybrid model comprising feedforward and energy-based blocks accounting for digital and analog circuits. We derive a novel algorithm to compute gradients end-to-end in ff-EBMs by backpropagating and “eq-propagating'' through feedforward and energy-based parts respectively, enabling EP to be applied to much more flexible and realistic architectures. We experimentally demonstrate the effectiveness of the proposed approach on ff-EBMs where Deep Hopfield Networks (DHNs) are used as energy-based blocks. We first show that a standard DHN can be arbitrarily split into any uniform size while maintaining performance. We then train ff-EBMs on ImageNet32 where we establish new SOTA performance in the EP literature (46 top-1 \%). Our approach offers a principled, scalable, and incremental roadmap to gradually integrate self-trainable analog computational primitives into existing digital accelerators.
\end{abstract}

\section{Introduction}
Gradient-based optimization, the cornerstone and most energy greedy component of deep learning, fundamentally relies upon three factors: i) highly parallel digital hardware such as GPUs, ii) feedforward models and iii) backprop (BP). With skyrocketing demands of AI compute, cutting the energy consumption of AI systems, learning has become a economical, societal and environmental stake \citep{strubell2020energy} and calls for the exploration of novel compute paradigms \citep{thompson2020computational,scellier2021deep, stern2023learning}.

One promising path towards this goal is analog in-memory computing  \citep{sebastian2020memory}: when mapping weights onto a crossbar of resistive devices, Kirchoff current and voltage laws inherently achieve matrix-vector multiplications in constant time complexity \citep{cosemans2019towards}.
Stacking multiple such crossbars, an entire neural network can be mapped onto a physical system. An important formalism for such systems is that of \emph{energy-based} (EB) analog circuits \citep{kendall2020,stern2023physical,dillavou2023transistor, scellier2024fast} which are ``self-learning'' systems that compute loss gradients through two relaxations to equilibrium (i.e. two ``forward passes''), a procedure falling under the umbrella of energy-based learning (EBL) algorithms \citep{scellier2024energy}. 
One of these learning algorithms, Equilibrium Propagation (EP) \citep{scellier2017equilibrium}, particularly stands out with strong theoretical guarantees, relative scalability in the realm of backprop alternatives \citep{laborieux2022holomorphic, laborieux2023improving} and experimental demonstrations on small analog systems which are $10,000\times$ more energy-efficient and substantially faster than their GPU-based counterpart \citep{yi2023activity}. This suggests an alternative triad as a new compute paradigm for gradient-based optimization: i) analog hardware, ii) EBMs, iii) EP. 
\begin{wrapfigure}[22]{l}{0.4\linewidth}
    \begin{center}
    \includegraphics[width=0.5\textwidth]{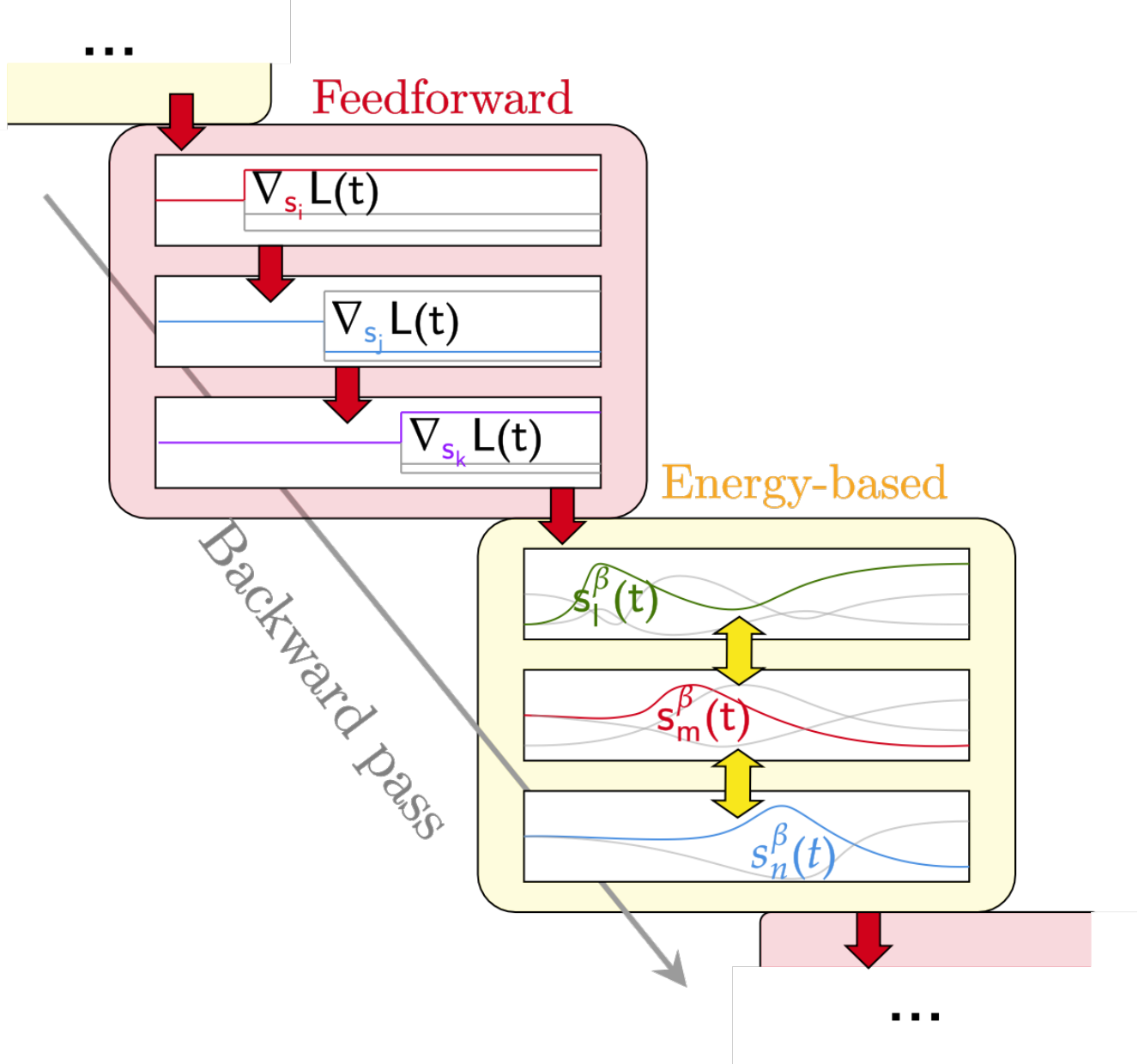}
    \end{center}
    \caption{Illustrating BP-EP backward gradient chaining through feedforward (red) and energy-based (yellow) blocks, accounting for digital and analog circuits respectively.}
\end{wrapfigure}
In this paper, we propose a theoretical framework to extend end-to-end gradient computation to a realistic setting where the system at use may or may not be \emph{fully} analog. Such a setting is plausible in the near term, due to by two major limitations. First, analog circuits exhibit many non-ideal physical behaviors which affect both the inference pathway \citep{wang2023digital, ambrogio2023analog} and parameter optimization \citep{nandakumar2020mixed, spoon2021toward, lammie2024lionheart} ,  in-turn compromising performance. Second, owing to the latency and energy-consumption of resistive devices' write operations, such analog circuits should be fully weight stationary -- weights must be written before the inference procedure begins -- which excludes many operations used conventionally in machine learning such as activation functions, normalization, attention \citep{spoon2021toward, jain2022heterogeneous, liu2023hardsea, li2023h3datten}. Therefore, analog systems are likely to be used in combination with auxiliary digital circuitry, resulting in hybrid mixed precision systems \citep{haensch2018next}. While the design of purely inferential engines made of analog and digital parts is nearing commercial maturity \citep{ambrogio2023analog}, \emph{in-situ} learning of such systems has barely been explored. One final challenge lies in proving EBL algorithms can scale in a manner comparable to backprop, given the requirement of \emph{simulating} EB systems on GPUs. Given the necessity of convergence, this amounts in practice in performing lengthy root finding algorithms to simulate physical equilibrium, limiting proof-of-concepts thereof to relatively shallow models \citep{scellier2024energy, scellier2024fast}.

Our work contends that the best of both worlds can be achieved with the following triad: i) hybrid digital \emph{and} analog hardware, ii) feedforward \emph{and} EB models, iii) BP \emph{and} EP. Namely, by modelling digital and analog parts as feedforward and EB modules respectively, the core contribution of our paper is to show how backprop and EP error signals can be chained end-to-end through feedforward and EB blocks respectively in a principled fashion. Rather than opposing digital and analog, or backprop and ``alternative'' learning algorithms as often done in the literature, we propose a novel hardware-aware building block which can, in principle, leverage advances from \emph{both} digital and analog hardware in the near-term. More specifically:

\begin{itemize}
    \item We propose \emph{Feedforward-tied Energy-based Models} (ff-EBMs, Section~\ref{subsec:ff-ebm}) as high-level models of mixed precision systems whose inference pathway read as the composition of feedforward and EB modules (Eq.~(\ref{def:deeply-nested-model}), Alg.~\ref{alg:inference-ff-ebm}). 
    \item We show that gradients in ff-EBMs can be computed in an end-to-end fashion (Section~\ref{subsec:chain-ep-bp}), backpropagating through feedforward blocks and ``eq-propagating'' through EB blocks (Theorem~\ref{theorem:main-result}, Alg.~\ref{alg:implicit-bp-ep-chain}) and that this procedure is rooted in a deeply-nested optimization problem (Section~\ref{subsec:deeply-nested}). 
    \item Finally, we experimentally demonstrate the effectiveness of our algorithm on ff-EBMs where EBM blocks are Deep Hopfield Networks (DHNs) (Section~\ref{sec:exps}). We show that i) gradient estimates computed by our algorithm (Alg.~\ref{alg:implicit-bp-ep-chain}) near perfectly match gradients computed by end-to-end automatic differentiation (Section~\ref{subsec:gdd}), ii) a standard DHN model can be arbitrarily split into a ff-DHN with the equivalent layers and architectural  layers while maintaining performance and remaining on par with automatic differentiation (Section~\ref{subsec:split-exp}), iii)  the proposed approach yields 46 \% top-1 (70\% top-5) validation accuracy on ImageNet32 when training a ff-EBM of 16 layers, thereby significantly beating EP current performance state-of-the-art by a large margin without using holomorphic transformations inside EBM blocks \citep{laborieux2022holomorphic, laborieux2023improving} 
\end{itemize}

\section{Background}

\paragraph{Notations.} Denoting $A: \mathbb{R}^n \to \mathbb{R}^m$ a differentiable mapping, we denote its \emph{total} derivative with respect to $s_j$ as $d_{s_j} A(s) := d A(s)/ ds_j \in \mathbb{R}^m$, its \emph{partial} derivative with respect to $s_j$ as $\partial_{j}A(s) := \partial A(s)/ \partial s_j \in \mathbb{R}^m$. When $A$ takes scalar values ($m=1$), its \emph{gradient} with respect to $s_j$ is denoted as $\nabla_j A(s) := \partial_j A(s)^\top$.

\subsection{Energy-based models (EBMs)}
For a given static input and set of weights, Energy-based models (EBMs) implicitly yield a  prediction through the minimization of an energy function -- as such they are a particular kind of implicit model. Namely, an EBM is defined by a (scalar) energy function $E: s, \theta, x \to E(s, \theta, x) \in \mathbb{R}$ where $x$, $s$, and $\theta$ respectively denote a static input, hidden and output neurons and model parameters, and each such tuple defines a configuration with an associated scalar energy value. Amongst all configurations for a given input $x$ and some model parameters $\theta$, the model prediction $s_\star$ is implicitly given as an equilibrium state which minimizes the energy function:
\begin{equation}
    s_\star := \arg \min_s E(s, \theta, x).
    \label{def:ebm}
\end{equation}
\subsection{Standard bilevel optimization}
Assuming that $\nabla_s^2 E(x, s_\star, \theta)$ is invertible, note that the equilibrium state $s_\star$ implicitly depends on $x$ and $\theta$ by virtue of the implicit function theorem \citep{dontchev2009implicit}. Therefore our goal when training an EBM, for instance in a supervised setting, is to adjust the model parameters $\theta$ such that $s_\star(x, \theta)$ minimizes some cost function $\ell: s, y \to \ell(s, y) \in \mathbb{R}$ where $y$ is some ground-truth label associated to $x$. More formally, this learning objective can be stated with the following \emph{bilevel optimization problem} \citep{zucchet2022beyond}:
\begin{equation}
    \min_{\theta}\mathcal{C}(x, \theta, y) := \ell(s_\star, y) \quad \textrm{s.t.} \quad  s_\star = \arg \min_s E(s, \theta, x) \nonumber.
    \label{def:bilevel-problem}
\end{equation}
Solving Eq.~(\ref{def:bilevel-problem}) in practice amounts to computing the gradient of its outer objective $\mathcal{C}(x, \theta)$ with respect to $\theta$ ($d_\theta \mathcal{C}(x, \theta)$) and then perform gradient descent over $\theta$. 

\subsection{Equilibrium Propagation (EP)}
An algorithm used to train an EBM model in the sense of Eq.~(\ref{def:bilevel-problem}) may be called an EBL algorithm \citep{scellier2024energy}. Equilibrium Propagation (EP) \citep{scellier2017equilibrium} is an EBL algorithm which computes an estimate of $d_\theta \mathcal{C}(x, \theta)$ with at least two phases. During the first phase, the model is allowed to evolve freely to $s_\star = \arg \min_s E(s, \theta, x)$. Then, the model is slightly nudged towards decreasing values of cost $\ell$ and settles to a second equilibrium state $s_\beta$. This amounts to augment the energy function $E$ by an additional term $\beta \ell(s, y)$ where $\beta \in \mathbb{R}^\star$ is called the \emph{nudging factor}. Then the weights are updated to increase the energy of $s_\star$ and decrease that of $s_\beta$, thereby ``contrasting'' these two states. More formally,
\cite{scellier2017equilibrium} prescribe in the seminal EP paper: 
\begin{equation}
    s_\beta := \arg \min_s \left[ E(s, \theta, x) + \beta\ell(s, y)\right], \quad 
    \Delta \theta ^{\rm EP} := \frac{\alpha}{\beta} \left(\nabla_2 E(s_\star, \theta, x) - \nabla_2 E(s_\beta, \theta, x) \right),
    \label{def:pos-ep-weight-update}
\end{equation}

where $\alpha$ denotes some learning rate. EP comes in different flavours depending on the sign of $\beta$ inside Eq.~(\ref{def:pos-ep-weight-update}) or on whether two nudged states of opposite nudging strengths ($\pm \beta$) are contrasted, a variant called \emph{Centered} EP (C-EP) which was shown to work best in practice \citep{laborieux2021scaling, scellier2024energy} and reads as:

\begin{equation}
    \Delta \theta ^{\rm C-EP} := \frac{\alpha}{2\beta} \left(\nabla_2 E(s_{-\beta}, \theta, x) - \nabla_2 E(s_\beta, \theta, x) \right),
    \label{def:c-ep-weight-update}
\end{equation}

\section{Tying energy-based models with feedforward blocks}
This section mirrors the background section by introducing a new model, the naturally associated optimization problem and a new learning algorithm. We first introduce \emph{Feedforward-tied EBMs} (ff-EBMs, section~\ref{subsec:ff-ebm}) which read as composition of feedforward and EB transformations (Alg.~\ref{alg:inference-ff-ebm}). We then show how optimizing ff-EBMs amounts to solving a multi-level optimization problem (Section~\ref{subsec:deeply-nested}) and propose a BP-EP gradient chaining algorithm as a solution (Section~\ref{subsec:chain-ep-bp}, Theorem~\ref{theorem:main-result}, Alg.~\ref{alg:implicit-bp-ep-chain}). We highlight as an edge case that ff-EBMs reduce to standard feedforward nets (Lemma~\ref{lma:ff}) and the proposed BP-EP gradient chaining algorithm to standard BP (Corollary~\ref{cor:ff-bp}) when each EB block comprises a single hidden layer.

\subsection{Feedforward-tied Energy-based Models (ff-EBMs)}
\label{subsec:ff-ebm}

\paragraph{Inference procedure.} We define \emph{Feedforward-tied Energy-based Models} (ff-EBMs) as compositions of feedforward and EB transformations. Namely, an data sample $x$ is fed into the first feedforward transformation $F^1$ parametrized by some weights $\omega^1$, which yields an output $x^1_\star$. Then, $x^1_\star$ is fed as a static input into the first EB block $E^1$ with parameters $\theta^1$, which relaxes to an equilibrium state $s^1_\star$. $s^1_\star$ is in turn fed into the next feedforward transformation $F^1$ with weights $\omega^1$ and the above procedure repeats until reaching the output layer $\hat{o}$. More formally, denoting $F^k$ and $E^k$ the k$^{\rm th}$ feedforward and EB blocks parametrized by the weights $\omega^k$ and $\theta^k$ respectively, the inference pathway of a ff-EBM reads as:

\begin{align}
\left\{
    \begin{array}{l}
    s^0 := x \\
    x^k_\star := F^k(s^{k-1}_\star, \omega^k), \quad s^k_\star := \underset{s}{\arg\min} \ E^k(s, \theta^k, x^k_\star) \quad \forall k = 1 \cdots N-1\\
    \hat{o}_\star := F^N(s^{N-1}_\star, \omega^N)
    \end{array}
\right.
\label{def:deeply-nested-model}
\end{align}

ff-EBM inference procedure is depicted more compactly inside Fig.~\ref{fig:fwd-bwd} (left) and Alg.~\ref{alg:inference-ff-ebm}. 


\begin{minipage}[c]{0.46\linewidth}
\includegraphics[width=\linewidth]{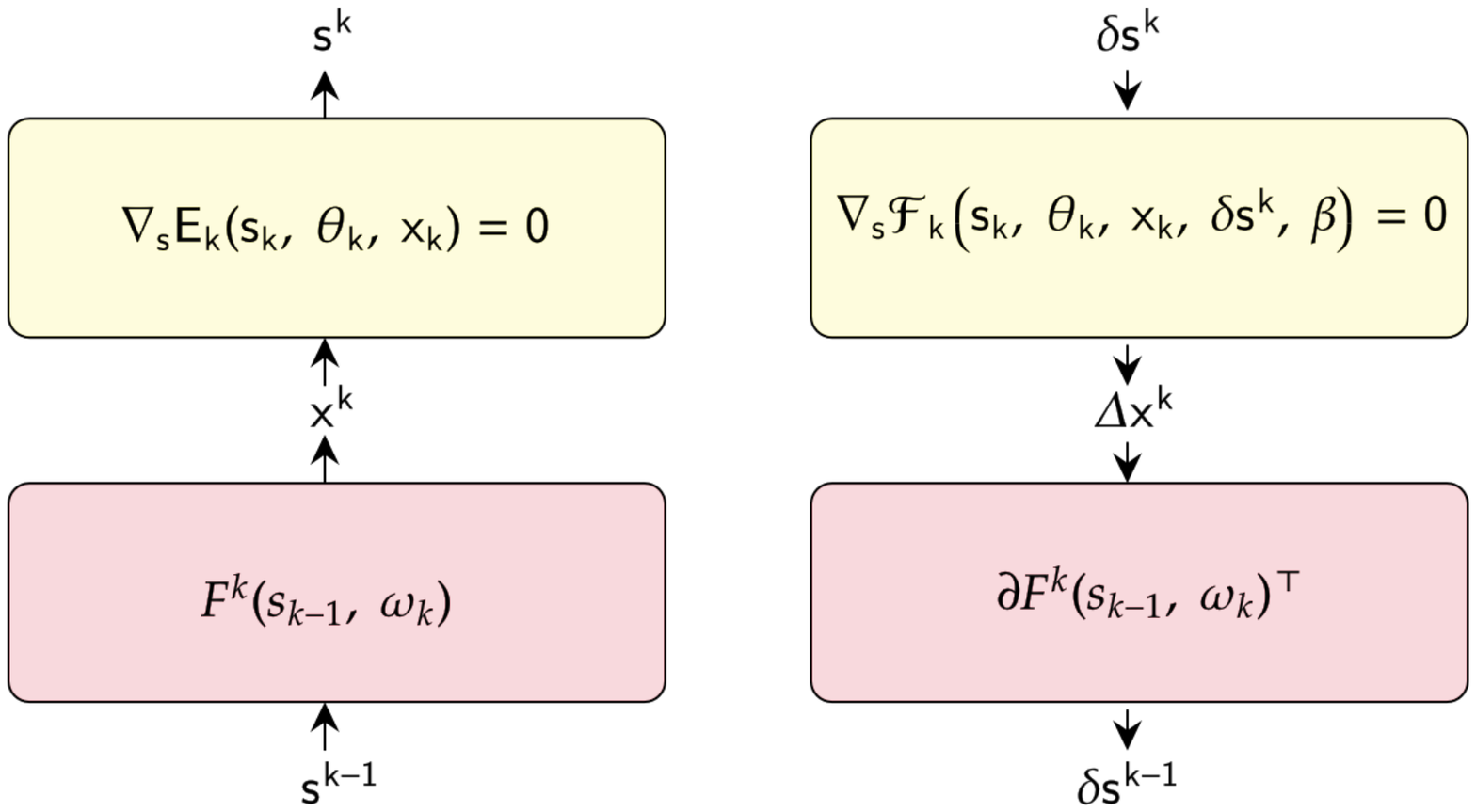}
\captionof{figure}{Depiction of the forward (left) and backward (right) pathways through a ff-EBM, with yellow and pink blocks denoting EB and feedforward transformations.}
\label{fig:fwd-bwd}
\end{minipage}
\hfill
\begin{minipage}[c]{0.5\linewidth}
\begin{algorithm}[H]
    \caption{ff-EBM inference (Eq.~(\ref{def:deeply-nested-model}))}
    \label{alg:inference-ff-ebm}
    \begin{algorithmic}[1]
    \State $s \gets x$
    \For{$k=1 \cdots N - 1$}
        \State $x \gets F^k\left(s, \omega^k\right)$
        \State $s \gets \underset{s}{\mathbf{Optim}}\left[E^k(s, \theta^k, x)\right]$
    \EndFor
    \State $\hat{o} \gets F^{N}(s, \omega^N)$
\end{algorithmic}
\end{algorithm}
\end{minipage}

\paragraph{Form of the energy functions.} We specify further the form of the energy of the k$^{\rm th}$ EB block of a ff-EBM as defined per Eq.~(\ref{def:deeply-nested-model}). The associated energy function $E^k$ takes some static input $x^k$ from the output of the preceding feedforward transformation, has hidden neurons $s^k$ and is parametrized by weights $\theta^k$ and more precisely defined as:

\begin{equation}
    E^k(s^k, \theta^k, x^k) := G^k(s^k) - s^{k^\top} \cdot x^k + U^k(s^k, \theta^k)
    \label{def:energy-block-k-version-1}
\end{equation}

Eq.~(\ref{def:energy-block-k-version-1}) reveals three different contributions to the energy. The first term determines the non-linearity applied inside the EB block \citep{zhang2017convergent, hoier2023dual}: for a given invertible and continuous activation function $\sigma$, $G$ is defined such that $\nabla G = \sigma^{-1}$ (see Appendix~\ref{subsec:details-model}).

The second term inside Eq.~(\ref{def:energy-block-k-version-1}) accounts for a purely feedforward contribution from the previous feedforward block $F^k$. Finally, the third term accounts for \emph{internal} interactions within the layers of the EB block. 

\paragraph{Recovering a feedforward net.} When taking the gradient of $E^k$ as defined in Eq.~(\ref{def:energy-block-k-version-1}) with respect to $s^k$ and zeroing it out, it can be seen that $s^k_\star$ is implicitly defined as:

\begin{equation}
    s^k_\star := \sigma\left(x^k - \nabla_1 U^k(s^k_\star, \theta^k) \right)
    \label{eq:steady-state-block-k}
\end{equation}

An interesting edge case highlighted by Eq.~(\ref{eq:steady-state-block-k}) is when $U^k = 0$ for all $k$'s, i.e. when there are no intra-block layer interactions, or equivalently when the EB block comprises a single layer only. In this case, $s_\star^k$ is simply a feedforward mapping $x^k$ through $\sigma$ and in turn the ff-EBM is simply a standard feedforward architecture (see Lemma~\ref{lma:ff} inside Appendix~\ref{app:ff-ebm}).

\subsection{Multi-level optimization of ff-EBMs}
\label{subsec:deeply-nested}

In the same way as learning EBMs can naturally be cast into a bilevel optimization problem, learning ff-EBMs can be inherently be mapped into a \emph{multi-level} optimization problem where the variables being optimized over in the inner subproblems are the EB block variables $s^1, \cdots, s^{N-1}$. To make this clearer, we re-write the energy function of the k$^{\rm th}$ block $E^k$ from Eq.~(\ref{def:energy-block-k-version-1}) to highlight the dependence between two consecutive EB block states:

\begin{equation}
    \widetilde{E}^k(s^k, \theta^k, s^{k-1}_\star, \omega^k) := E^k\left(s^k, \theta^k, F^k\left(s^{k-1}_\star, \omega^{k-1} \right)\right)
    \label{def:energy-block-k-version-2}
\end{equation}

It can be seen from Eq.~(\ref{def:energy-block-k-version-2}) that the equilibrium state $s^k_\star$ obtained by minimizing $E^k$ will be dependent upon the equilibrium state $s^{k-1}_\star$ of the previous EB block, which propagates back through prior EB blocks. Denoting $W := \{\theta^1, \cdots, \theta^{N-1}, \omega^1, \cdots, \omega^N\}$, the learning problem for a ff-EBM can therefore be written as:


\begin{align}
    \min_{W}\ &\mathcal{C}(x, W, y) :=\ell(\hat{o}_\star = F^N(s^{N-1}_\star, \omega^N), y) \label{def:multilevel-optim} \\
    \text{s.t.} &\quad s^{N-1}_\star = \underset{s}{\arg\min} \ \widetilde{E}^{N-1}(s, \theta^{N-1}, s^{N-2}_\star, \omega^{N-1}) \quad \cdots \quad \text{s.t.} \quad s^1_\star = \underset{s}{\arg\min} \ \widetilde{E}^1(s, \theta^1, x, \omega^1) \nonumber
\end{align}

Here again and similarly to bilevel optimization, solving Eq.~(\ref{def:multilevel-optim}) in practice amounts to computing $  g_{\theta^k} := d_{\theta^k} \mathcal{C}$ and  $g_{\omega^k} := d_{\omega^k} \mathcal{C}$ and perform gradient descent on $\theta^k$ and $\omega^k$.

\subsection{A BP--EP gradient chaining algorithm}
\label{subsec:chain-ep-bp}

\paragraph{Main result: explicit BP-EP chaining.} Based on the multilevel optimization formulation of ff-EBMs learning in Eq.~(\ref{def:multilevel-optim}), we state the main theoretical result of this paper in Theorem~\ref{theorem:main-result} (see proof in Appendix~\ref{app:proof-main-thm}). 

\begin{theorem}[Informal]
\label{theorem:main-result}
Assuming a model of the form Eq.~(\ref{def:deeply-nested-model}), we denote $s^1_\star, x^1_\star, \cdots, s^{N-1}_\star, \hat{o}_\star$ the states computed during the forward pass as depicted in Alg.~\ref{alg:inference-ff-ebm}. 
We define the nudged state of block $k$, denoted as $s^k_\beta$,  implicitly through $\nabla_1 \mathcal{F}^k(s^k_\beta, \theta^k, x^k_\star, \delta s^k, \beta) = 0$ with:
\begin{equation}
\mathcal{F}^k(s^k, \theta^k, x^k_\star, \delta s^k, \beta) := E^k(s^k, \theta^k, x^k_\star) + \beta s^{k^\top}\cdot \delta s^k
\end{equation}

Denoting $\delta s^k$ and $\Delta x^k$ the error signals computed at the input of the feedforward block $F^k$ and of the EB block $E^k$ respectively, then the following chain rule applies:

\begin{align}
&\delta s^{N-1} := \nabla_{s^{N-1}}\ell(\hat{o}_\star, y), \quad g_{\omega^N} = \nabla_{\omega^N} \ell(\hat{o}_\star, y) \label{eq:main-result:1}\\
&\forall k=2 \cdots N-1: \nonumber\\
&\left\{
    \begin{array}{l}
    \Delta x^k = \left.d_{\beta}\left(\nabla_{3}E^k (s^k_\beta, \theta^k, x^k_\star)\right)\right|_{\beta=0}, \quad g_{\theta^k} = \left.d_{\beta}\left(\nabla_{2}E^k (s^k_\beta, \theta^k, x^k_\star)\right)\right|_{\beta=0} \\
    \delta s^{k-1} = \partial_1 F^k\left(s^{k-1}_\star, \omega^k \right)^\top \cdot \Delta x^k, \quad g_{\omega^k} = \partial_2 F^k\left(s^{k-1}_\star, \omega^k\right)^\top \cdot \Delta x^k
    \end{array}
\right.
\label{eq:main-result:2}
\end{align}
\end{theorem}

\paragraph{Proposed algorithm: implicit BP-EP chaining.} Theorem~\ref{theorem:main-result} reads intuitively: it prescribes an \emph{explicit} chaining of EP error signals passing backward through $E^k$ ($\delta s^k \to \Delta x^k$) and BP error signals passing backward through $\partial F^{k^\top}$ ($\Delta x^k \to \delta s^{k-1}$), which directly mirrors the ff-EBM inference pathway as depicted in Fig.~\ref{fig:fwd-bwd}. Yet noticing that:


\begin{equation*}
\left\{
    \begin{array}{l}
   \delta s^{k-1} =  \partial_1 F^k\left(s^{k-1}_\star, \omega^k \right)^\top \cdot \Delta x^k = d_\beta \left.\left(\nabla_3 \widetilde{E}^{k}\left(s^{k}_\beta, \theta^{k}, s^{k-1}_\star, \omega^{k} \right)\right)\right|_{\beta=0}, \\
   g_{\omega^k} = \partial_2 F^k\left(s^{k-1}_\star, \omega^k \right)^\top \cdot \Delta x^k =  d_\beta \left.\left(\nabla_4 \widetilde{E}^{k}\left(s^{k}_\beta, \theta^{k}, s^{k-1}_\star, \omega^{k} \right)\right)\right|_{\beta=0},
    \end{array}
\right.
\end{equation*}

the same error signal can by passed through $\widetilde{E}^k$ ($\delta s^k \to \delta s^{k-1}$) where BP and EP are \emph{implicitly} chained inside $\widetilde{E}^k$ (see Appendix~\ref{app:proof-main-thm}). This insight, along with a centered scheme to estimate derivatives with respect to $\beta$ around 0 as done for the C-EP algorithm (Eq.~(\ref{def:c-ep-weight-update})), motivates the implicit BP-EP gradient chaining algorithm in Alg.~\ref{alg:implicit-bp-ep-chain} we used for our experiments (see Alg.~\ref{alg:explicit-bp-ep-chain} inside Appendix~\ref{app:explicit-bp-ep-chaining} for its explicit counterpart). For simplicity for here onwards and as the proposed algorithm appears to be a generalization of EP, we may refer to Alg.~\ref{alg:implicit-bp-ep-chain} as ``EP'' in the experimental section.

\begin{algorithm}[H]
    \caption{Implicit BP-EP gradient chaining (Theorem~(\ref{theorem:main-result}))}
    \label{alg:implicit-bp-ep-chain}
    \begin{algorithmic}[1]
    \State $\delta s, g_{\omega^N} \gets \nabla_{s^{N-1}}\ell(\hat{o}_\star, y), \nabla_{\omega^N} \ell(\hat{o}_\star, y) $ \Comment{Single backprop step}
    \For{$k=N - 1 \cdots 1$}
        \State $s_\beta \gets \underset{s}{\mathbf{Optim}}\left[\widetilde{E}^k(s, \theta^k, s^{k-1}_\star, \omega^k) + \beta s^\top \cdot \delta s\right]$ \Comment{EP through $\widetilde{E}^k$}
        \State $s_{-\beta} \gets \underset{s}{\mathbf{Optim}}\left[\widetilde{E}^k(s, \theta^k, s^{k-1}_\star, \omega^k) - \beta s^\top \cdot \delta s\right]$
        \State $ g_{\theta^k} \gets \frac{1}{2\beta}\left(\nabla_{2} \widetilde{E}^k(s_\beta, \theta^k, s^{k-1}_\star, \omega^k) - \nabla_{2} \widetilde{E}^k(s_{-\beta}, \theta^k, s^{k-1}_\star, \omega^k)\right)$
        \State $g_{\omega^k} \gets \frac{1}{2\beta}\left(\nabla_{4} \widetilde{E}^k(s_\beta, \theta^k, s^{k-1}_\star, \omega^k) - \nabla_{4} \widetilde{E}^k(s_{-\beta}, \theta^k, s^{k-1}_\star, \omega^k)\right)$ \Comment{i-BP through $F^k$}
    \State $ \delta s \gets \frac{1}{2\beta}\left(\nabla_{3} \widetilde{E}^k(s_\beta, \theta^k, s^{k-1}_\star, \omega^k) - \nabla_{3} \widetilde{E}^k(s_{-\beta}, \theta^k, s^{k-1}_\star, \omega^k)\right)$
    \EndFor
    \end{algorithmic}
\end{algorithm}

\paragraph{Recovering backprop.} When the ff-EBM under consideration is purely feedforward ($U^k=0$), we show that Eqs.~(\ref{eq:main-result:1})--(\ref{eq:main-result:2}) reduce to standard BP through a feedforward net (Corollary~\ref{cor:ff-bp}, Alg.~\ref{alg:implicit-bp-ep-chain-ff} and Alg.~\ref{alg:explicit-bp-ep-chain-ff} in Appendix~\ref{app:proof-main-thm}). Therefore, since this case is extremely close to standard BP through feedforward nets, we do not consider this setting in our experiments.

\section{Experiments}
\label{sec:exps}
In this section, we first present the ff-EBMs at use in our experiments (Section~\ref{subsec:exp-setup}) and carry out \emph{static} gradient analysis -- computing and analyzing ff-EBM parameter gradients for some $x$ and $y$ (Section~\ref{subsec:gdd}). We extend the observation made by \cite{ernoult2019updates} to ff-EBMs that \emph{transient} EP parameter gradients throughout the second phase match those computed by automatic differentiation through equilibrium and across blocks (Fig.~(\ref{fig:gdd})), with resulting final gradient estimates near perfectly aligned (Fig.~\ref{fig:bars}). We then show on the CIFAR-10 task that performance of ff-EBMs can be maintained across various block splits and on par with automatic differentiation while keeping the same number of layers (Section~\ref{subsec:split-exp}). Finally, we perform further ff-EBM training experiments on CIFAR-100 and ImageNet32 where we establish a new performance state-of-the-art in the EP literature (Section~\ref{subsec:scaling-exp}). 
\subsection{Setup}
\label{subsec:exp-setup}
\paragraph{Model.} Using the same notations as in Eq.~(\ref{def:energy-block-k-version-1}), the ff-EBMs at use in this section are defined through: 
\begin{equation}
\left\{
    \begin{array}{l}
U^k_{\rm FC}(s^k, \theta^k) := -\frac{1}{2}s^{k\top} \cdot \theta^k \cdot s^k, \\
U^k_{\rm CONV}(s^k, \theta^k) := -\frac{1}{2}s^k \bullet \left(\theta^k \star s^k\right)   
    \end{array}
\right. , \quad F^k(s^{k-1}, \omega^k) := {\rm BN}\left(\mathcal{P}\left(\omega^k_{\rm CONV} \star s^{k-1}_L\right); \omega^k_\alpha, \omega^k_\beta\right)
\end{equation}
with ${\rm BN}(\cdot; \omega_\alpha^k  \omega_\beta^k)$, $\mathcal{P}$ and $\star$ the batchnorm, pooling and convolution operations, $\bullet$ the generalized dot product for tensors and $s^k := \left(s^{k^\top}_1, \cdots s^{k^\top}_L\right)^\top$ the state of block $k$ comprising $L$ layers. The EBM blocks are usually called Deep Hopfield Networks (DHNs) and the weight matrix $\theta^k$ is symmetric and has a sparse, block-wise structure such that each layer $s^k_\ell$ is bidirectionally connected to its neighboring layers $s^k_{\ell - 1}$ and $s^k_{\ell + 1}$ through connections $\theta^k_{\ell -1}$ and $\theta^{k^\top}_{\ell}$ respectively (see Appendix~\ref{subsec:details-model}), either with fully connected ($U^k_{\rm FC}$) and convolutional operations ($U^k_{\rm CONV}$). Finally, the non-linearity $\sigma$ applied within EB blocks is $\sigma(x) := \min\left(\max\left(\frac{x}{2}, 0\right), 1\right)$. 

\paragraph{Equilibrium computation.} As depicted in Alg.~\ref{alg:implicit-bp-ep-chain}, the steady states $s_{\pm \beta}$ may be computed with any loss minimization algorithm. Here and as done in most past works on EP \citep{ernoult2019updates, laborieux2021scaling, laborieux2022holomorphic, scellier2024energy}, we employ a fixed-point iteration scheme to compute the EB blocks steady states. Namely, we iterate Eq.~(\ref{eq:steady-state-block-k}) until reaching equilibrium (the same scheme is used for ff-EBM inference, Alg.~\ref{alg:inference-ff-ebm}, with $\beta=0$.):
\begin{equation}
    s^k_{\pm\beta, t + 1} \gets \sigma\left(x^k - \nabla_1 U^k(s^k_{\pm\beta, t}, \theta^k) \mp \beta \delta s^k \right)
    \label{eq:fixed-point-iteration}
\end{equation}
We employ a scheme to \emph{asynchronously} update even ($s^k_{2\ell'}$) and odd ($s^k_{2\ell' + 1}$) layers \citep{scellier2024energy} -- see Appendix~\ref{subsec:details-model}. 

\paragraph{Algorithm baseline.} As an algorithmic baseline, we simply use automatic differentiation (AD) backward through the fixed-point iteration scheme Eq.~(\ref{eq:fixed-point-iteration}) with $\beta=0$ and directly initializing $s^k_{t=0} = s_\star$. This version of AD, where we \emph{backpropagate through equilibrium}, is known as ``Recurrent Backpropagation'' \citep{almeida1987learning, pineda1987generalization} or Implicit Differentiation (ID). 

\subsection{Static comparison of EP and ID on ff-EBMs}
In order to study the \emph{transient dynamics} of ID and EP, we define, with $W^k := \{\theta^k, \omega^k\}$:
\begin{equation}
\left\{
    \begin{array}{l}
    \widehat{g}^{\rm ID}_{W^k}(t) := \sum_{k=0}^T d_{W^k(T - k)} \mathcal{C}(x, W, y), \\
    \widehat{g}^{\rm EP}_{W^k}(t) := \frac{1}{2\beta}\left(\nabla_{W^k}\widetilde{E}^k(s^k_\beta(t), W^k, s^{k-1}_\star) - \nabla_{W^k}\widetilde{E}^k(s^k_{-\beta}(t), W^k, s^{k-1}_\star)\right),
    \end{array}
\right.
\end{equation}
where $s^{\pm \beta}(t)$ is computed from Eq.~(\ref{eq:fixed-point-iteration}) with the nudging error current $\delta s^k$ computed with Alg.~\ref{alg:implicit-bp-ep-chain}, and $T$ is the total number of iterations used for both ID and EP in the gradient computation phase.
\label{subsec:gdd}
\begin{figure}[ht!]
    \begin{center}
    \includegraphics[width=\textwidth]{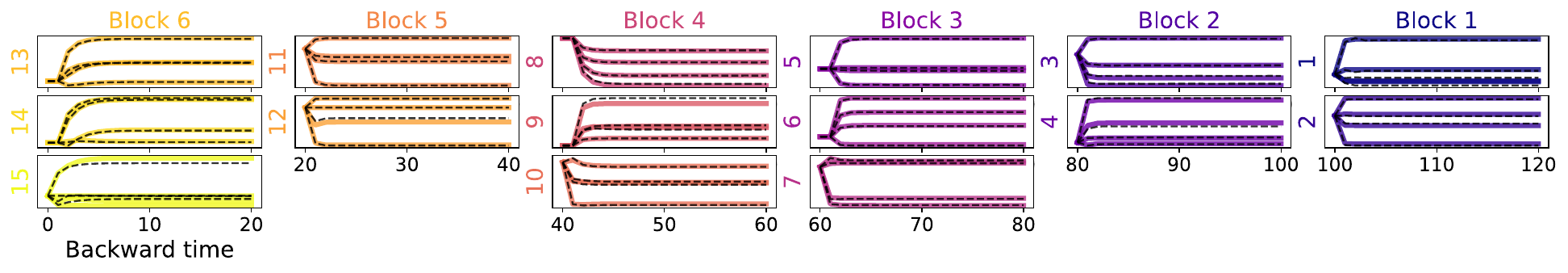}
    \end{center}
    \caption{EP and ID partially computed gradients ($(\widehat{g}_{w}^{\rm EP}(t))_{t \geq 0}$ in black dotted curves and $(\widehat{g}_{w}^{\rm ID}(t))_{t \geq 0}$ in plain colored curves) going \emph{backward through equilibrium} for ID and \emph{forward through the nudging phase} for EP \citep{ernoult2019updates} for a random sample $x$ and associated label $y$. The ff-EBM comprises 6 blocks and 15 layers in total, with block sizes of either 2 or 3 layers. Each subpanel represents a layer (labelled on the y-axis) with each curve corresponding to a randomly selected weight. ``Backward'' time is indexed from $t=0$ to $T=120$, starting from block 6 backward to block 1,  with 20 fixed-point iteration dynamics (Eq.~(\ref{eq:fixed-point-iteration})) being used for both EP and ID within each EB block.}
\label{fig:gdd}
\end{figure}

 For a given block $k$, $d_{W^k(T - k)} \mathcal{C}(x, W, y)$ is the ``sensitivity'' of the loss $\mathcal{C}$ to parameter $W^k$ at timestep $T - k$ so that $\widehat{g}^{\rm AD}_{W^k}(t)$ is a ID gradient \emph{truncated} at $T-t$. Similarly, $\widehat{g}^{\rm EP}_{W^k}(t)$ is an EP gradient truncated at $t$ steps forward through the nudged phase. When $T$ is sufficiently large, $\widehat{g}^{\rm ID}_{W^k}(T)$ and $\widehat{g}^{\rm EP}_{W^k}(T)$ converge to $d_{W^k}\mathcal{C}(x, W, y)$. Fig.~\ref{fig:gdd} displays $(\widehat{g}^{\rm ID}_{W^k}(t))_{t \geq 0}$ and $(\widehat{g}^{\rm EP}_{W^k}(t))_{t \geq 0}$ on an heterogeneous ff-EBM of 6 blocks and 15 layers with blocks comprising 2 or 3 layers for a randomly selected sample $x$ and its associated label $y$ -- see caption for a detailed description. It can be seen EP and ID error weight gradients qualitatively match very well throughout time, across layers and blocks. More quantitatively, we display the cosine similarity between the final EP and ID weight gradient estimate $\widehat{g}^{\rm ID}_{W^k}(T)$ and $\widehat{g}^{\rm EP}_{W^k}(T)$ for each layer and observe that EP and ID weight gradients are near perfectly aligned.

\subsection{Splitting experiment}
\label{subsec:split-exp}

For a given EBM (standard, single block) and a \emph{fixed} number of layers, we ask whether block splitting of this EBM into a ff-EBM with multiple EB blocks affects training performance. We address this question with two different depths ($L=6$ and $L=12$ layers in total) and various block splits maintaining a total number of layers (e.g. for $L=6$,  1 block of 6 layers, 2 blocks of 3 layers, etc.) and display the results obtained on the CIFAR-10 task inside Table~\ref{table:split}. We observe that the performance achieved by EP on the 6-layers deep EBM is maintained across 4 different block splits between 89\% and 90\% and is consistently on par with the ID baseline for each ff-EBM and with the literature on EBMs of same depth \citep{scellier2024energy, laborieux2022holomorphic}. Similarly, we observe that the performance achieved by EP on ff-EBMs with a total of 12 layers is maintained around $92\%$ with three different block sizes, matches ID performance on each of these and surpasses EP state-of-the art on CIFAR-10 \citep{scellier2024energy}. Overall these results suggest the agnosticity of ff-EBMs to EB block sizes and therefore appear to be flexible in design. 

\begin{figure}[ht!]
\begin{minipage}[c]{0.46\linewidth}
\includegraphics[width=\linewidth]{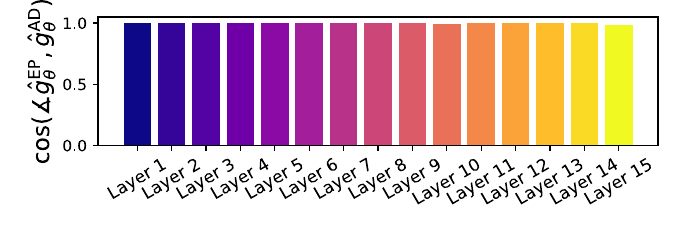}
\captionof{figure}{Cosine similarity between EP and ID weight gradients on a randomly selected sample $x$ and associated label $y$ in the same setting as Fig.~\ref{fig:gdd} using the same color code to label the layers. We observe near-perfect alignment between EP and ID gradients.}
\label{fig:bars}
\end{minipage}
\hfill
\begin{minipage}[c]{0.5\linewidth}
\centering
\captionof{table}{Validation accuracy and Wall Clock Time (WCT) obtained on CIFAR-10 by EP (Alg.~\ref{alg:implicit-bp-ep-chain}) and ID on models with different number of layers ($L$) and block sizes (``bs''). 3 seeds are used. }
  \begin{tabular}{lcccc}
\toprule
& \multicolumn{2}{c}{EP} & \multicolumn{2}{c}{ID} \\
& Top-1 (\%) & WCT & Top-1 (\%) & WCT \\
\cmidrule(r){2-5}
{\bfseries L =6} & & & & \\
bs=6 & 88.8$^{\pm 0.2}$  &  8:06 & 87.3 $^{\pm0.6} $  & 8:05  \\
bs=3 &  89.5$^{\pm 0.2}$ &  8:01 & 89.2$^{\pm 0.2}$ & 7:40  \\
bs=2 &  90.1$^{\pm 0.2}$  &  7:47& 90.0 $^{\pm 0.2}$   & 7:18 \\
\cmidrule(r){2-5}
{\bfseries L =12} & & & & \\
bs=4 & 91.6$^{\pm 0.1}$  &  7:49 & 91.6$^{\pm 0.1}$ & 7:08  \\
bs=3 &  {\bfseries 92.2}$^{\pm 0.2}$  &  6:06 & 92.2$^{\pm 0.1}$ & 5:59  \\
bs=2 & 91.7$^{\pm 0.2}$  &  6:10& 91.8$^{\pm 0.1}$ & 6:08  \\
\bottomrule
\label{table:split}
\end{tabular}
\end{minipage}
\end{figure}


\subsection{Scaling experiment}
\label{subsec:scaling-exp}
Finally, unlike Section~\ref{subsec:split-exp}, we now consider ff-EBMs of fixed block size 2 and train them with two different depths ($L=12$ and $L=15)$ on CIFAR-100 and ImageNet32 by EP and ID and show the results obtained in Table~\ref{tab:scaling}. Here again we observe that EP matches ID performance on all models and tasks, ff-EBMs benefit from depth, and the performance obtained by training the 16-layers deep ff-EBM by EP exceeds state-of-the-art performance on ImageNet32 by around 10\% top-1 validation accuracy \citep{laborieux2022holomorphic} and by around 5\% the best performance reported on this benchmark among backprop alternatives \citep{hoier2023dual}. 

\begin{table}[ht!]
\centering
\caption{Validation accuracy and Wall Clock Time (WCT) obtained on CIFAR100 and ImageNet32 by EP and Autodiff on models with different number of layers ($L$) and a block size of 2 (bs=2). Experiments are ran across 3 different seeds. We compare our results against best published results on ImageNet32 by EP \citep{laborieux2022holomorphic} and amongst all backprop alternatives \citep{hoier2023dual}.}
\begin{tabular}{llcccccc}
\toprule
& & \multicolumn{3}{c}{EP} & \multicolumn{3}{c}{ID} \\
& & Top-1 (\%) & Top-5 (\%) &  WCT & Top-1 (\%) & Top-5 (\%) & WCT \\
\cmidrule(r){3-8}
\multirow{2}{*}{CIFAR100}& L=12 & 69.3 $^{\pm 0.2}$ & 89.9 $^{\pm 0.5}$& 4:33 & 69.2$^{\pm 0.1}$ & 90.0 $^{\pm 0.2}$ &  4:16\\
&L=15 & 71.2$^{\pm 0.2}$ & 90.2$^{\pm 1.2}$ &  2:54 & 71.1$^{\pm 0.3}$ & 90.9 $^{\pm 0.1}$  &  2:44 \\
\cmidrule(r){1-2}
\multirow{2}{*}{ImageNet32}&L=12 & 44.7 $^{\pm 0.1}$ & 61:00 $^{\pm 0.1}$ &65:23 &44.7 $^{\pm 0.6}$ & 68.9$^{\pm 0.6}$ &  57:00 \\
&L=15 & {\bfseries 46.0} $^{\pm 0.1}$ & {\bfseries 70.0} $^{\pm 0.2}$& 46:00 & 45.5 $^{\pm 0.1}$ &  69.0 $^{\pm 0.1}$  & 40:01 \\
\midrule
\midrule
\multicolumn{2}{c}{\cite{laborieux2022holomorphic}}& 36.5 & 60.8& --& -- & --  & -- \\
\multicolumn{2}{c}{\cite{hoier2023dual}}& 41.5 & 64.9& --& -- & --  & -- \\
\bottomrule
\end{tabular}
\label{tab:scaling}
\end{table}

\section{Discussion}

\paragraph{Related work. } Since fixed-point iteration schemes were proposed to facilitate EP experiments \citep{ernoult2019updates, laborieux2021scaling}, there is a growing body of work revolving around algorithmic extensions of EP and assessments of its scalability on vision tasks. Most notably, \cite{laborieux2022holomorphic} introduced a holomorphic version of EP where loss gradients are computed with adiabatic oscillations of the model through nudging in the complex plane, and
was very recently extended to more general implicit models \citep{laborieux2023improving}. Moving further towards physical implementations of EP, \cite{scellier2022agnostic} proposed a fully black-box version of EP where details about the system may not be known. All these advances could be readily applied inside our EP-BP chaining algorithm to EB blocks. The work closest to ours, albeit with a purely theoretical motivation and without clear algorithmic prescriptions, is that of \cite{zach2021bilevel} where feedforward model learning is cast into a deeply nested optimization where consecutive layers are tied by elemental pair-wise energy functions, which more recently inspired the Dual Propagation algorithm \citep{hoier2023dual}. This setting can be construed as a particular case of ff-EBM learning by EP where each EB block comprises a \emph{single} layer ($U^k = 0$ inside Eq.~(\ref{def:energy-block-k-version-1}) which, however, remains extremely similar to BP (see last paragraph of Section~\ref{subsec:chain-ep-bp}).

\paragraph{Limitations and future work.} Since our recipe advocates EP--BP chaining by construction, it is fair to say that ff-EBM learning partially inherits the pitfalls of BP.
Fortunately, nothing prevents feedforward modules inside ff-EBMs to be trained by \emph{any} BP alternative to mitigate specific issues. For instance: BP can be parameterized by feedback weights to obviate weight transport from the inference circuit to the gradient computation circuit \citep{akrout2019deep}; BP gradients can be approximated as finite differences of feedback operators \citep{ernoult2022towards}; or computed via implicit forward-mode differentiation by applying random weight perturbations in the inference circuit \citep{hiratani2022stability, fournier2023can, malladi2023fine}; local layer-wise self-supervised or supervised loss functions can be used to prevent ``backward locking'' \citep{belilovsky2019greedy,ren2022scaling,  hinton2022forward}. This insight may help exploring many variants of ff-EBM training.


Pursuing the core motivation of this work, one natural extension of this study is to incorporate \emph{more hardware realism into ff-EBMs}. Beyond Deep Hopfield networks, Deep Resistive Nets (DRNs) -- concurrently developed by \cite{scellier2024fast} and strongly inspired by \cite{kendall2020} -- are exact models of idealized analog circuits, are fast to simulate and were shown to be trainable by EP. As such, using DRNs as EB blocks inside ff-EBMs is an exciting research direction. Yet, going further into analog hardware modeling for ff-EBMs comes with new challenges when taking into account device non-idealities which may affect the inference pathway, such as analog-to-digital and digital-to-analog noise \citep{rasch2023hardware, lammie2024lionheart}.

Finally, considerable work is needed to prove ff-EBM further at scale on more difficult tasks (e.g. standard ImageNet), considerably deeper architectures and beyond vision tasks. One other exciting research direction would be the design of \emph{ff-EBM based transformers}, with attention layers being chained with energy-based fully connected layers inside attention blocks.

\paragraph{Concluding remarks and broader impact.} We show that ff-EBMs constitute a novel framework for deep-learning in heterogeneous hardware settings. We hope that the  algorithm proposed can help to move beyond the typical division between digital \emph{versus} analog or BP \emph{versus} BP-free algorithms and that the greater energy-efficiency afforded by this framework provides a more pragmatic, near-term blueprint to mitigate the dramatic carbon footprint of AI training \citep{strubell2020energy}. Being still a long way from fully analog training accelerators at commercial maturity, we believe this work offers an incremental and sustainable roadmap to gradually integrate analog, energy-based computational primitives as they are developed into existing digital accelerators. 

\clearpage

\section*{Acknowledgements and disclosure of funding}
The authors warmly thank Irina Rish, Jack Kendall and Suhas Kumar for their support of the project idea from the very start, as well as Gregory Kollmer and Mohammed Fouda for useful feedback on the manuscript. TN acknowledges the support from the Canada Excellence Research Chairs Program, as well as CIFAR and Union Neurosciences et Intelligence Artificielle Quebec (UNIQUE). This research was enabled by the computational resources provided by the Summit supercomputer, awarded through the Frontier DD allocation and INCITE 2023 program for the project "Scalable Foundation Models for Transferable Generalist AI" and SummitPlus allocation in 2024. These resources were supplied by the Oak Ridge Leadership Computing Facility at the Oak Ridge National Laboratory, with support from the Office of Science of the U.S. Department of Energy. ME acknowledges funding from Rain AI which commercializes technologies based on brain-inspired learning algorithms, as well as Constance Castres Saint-Martin for her unwavering support at the maternity hospital where most of this manuscript was written. 

\bibliographystyle{abbrvnat}
\bibliography{biblio}

\newpage

\appendix
\begin{appendices}
\section{Appendix}
\DoToC

\subsection{Further insights on ff-EBMs}
\label{app:ff-ebm}

In this section:

\begin{itemize}
    \item We formally define \emph{Feedforward-tied Energy-based Models} (ff-EBMs) with precise assumptions on the energy-based and feedforward blocks (Def.~\ref{def:ff-EBM-formal}).
    \item We show that when energy-based blocks comprise a single layer only, the ff-EBM becomes purely feedforward (Lemma~\ref{lma:ff}).
\end{itemize}

\begin{definition}[ff-EBMs]
\label{def:ff-EBM-formal}
A Feedforward-tied Energy-based Model (ff-EBM) of size $N$ comprises $N$ twice differentiable feedforward mapping $F^1, \cdots, F^N$ and $N-1$ twice differentiable energy functions $E^1, \cdots, E^{N-1}$. For a given $x$, the inference procedure reads as:
    \begin{align}
\left\{
    \begin{array}{l}
    s^0 := x \\
    x^k_\star := F^k(s^{k-1}_\star, \omega^k), \quad s^k_\star := \underset{s}{\arg\min} \ E^k(s, \theta^k, x^k_\star) \quad \forall k = 1 \cdots N-1\\
    \hat{o}_\star := F^N(s^{N-1}_\star, \omega^N)
    \end{array}
\right.
\label{def:ff-EBM-inference-formal}
\end{align}

Finally, we assume that $\forall k = 1 \cdots N-1$, $\nabla_1^2 E^k(s^k_\star, \theta^k, \omega^k)$ is invertible.
\end{definition}

\begin{lemma}
\label{lma:ff}
We consider ff-EBM per Def.~(\ref{def:ff-EBM-formal}) where the energy functions $E^k$ have the form:

\begin{equation}
    E^k(s^k, \theta^k, x^k) := G^k(s^k) - s^{k^\top} \cdot x^k + U^k(s^k, \theta^k).
    \label{ff:energy-block-k-formal}
\end{equation}

We assume that $U^k = 0$ for $k = 1 \cdots N-1$, $s \to \nabla G(s)$ is invertible and we denote $\sigma := \nabla G^{-1}$. Then, the resulting model is a feedforward model described by the following recursive equations:
    \begin{align}
\left\{
    \begin{array}{l}
    s^0_\star = x \\
    x^k_\star = F^k(s^{k-1}_\star, \omega^k), \quad s^k_\star = \sigma(x^k_\star) \quad \forall k = 1 \cdots N-1\\
    \hat{o}_\star := F^N(s^{N-1}_\star, \omega^N)
    \end{array}
\right.
\label{def:deeply-nested-model:feedforward-case}
\end{align}
\end{lemma}

\begin{proof}[Proof of Lemma~\ref{lma:ff}]
    Let $k \in [1, N-1]$. By definition of $s^k_\star$ and $x^k_\star$:

    \begin{align}
        &\nabla_1 E^k(s^k_\star, \theta^k, x^k_\star) = 0 \nonumber \\
    \Leftrightarrow \quad & \nabla G^k(s^k_\star) - x^k_\star + \nabla_1 U^k(s^k_\star, \theta^k) = 0 \nonumber \\
    \Leftrightarrow \quad & s^k_\star = \sigma \left(x^k_\star - \nabla_1 U^k(s^k_\star, \theta^k) \right) \label{ff:result}
    \end{align}

    Therefore Eq.~(\ref{def:deeply-nested-model:feedforward-case}) is immediately obtained from Eq.~(\ref{ff:result}) with $U^k=0$.
    
\end{proof}
\newpage
\subsection{Proof of Theorem~\ref{theorem:main-result}}
\label{app:proof-main-thm}

The proof of Theorem~\ref{theorem:main-result} is structured as follows:

\begin{itemize}
    \item We directly solve the multilevel problem optimization defined inside Eq.~(\ref{def:multilevel-optim}) using a Lagrangian-based approach (Lemma~\ref{lma:lagrangian}), yielding optimal Lagrangian multipliers, block states and loss gradients.
    \item We show that by properly nudging the blocks, EP implicitly estimates the previously derived Lagrangian multipliers (Lemma~\ref{lma:lagrangian-ep}).
    \item We demonstrate Theorem~\ref{theorem:main-result} by combining Lemma~\ref{lma:lagrangian} and Lemma~\ref{lma:lagrangian-ep}.
    \item Finally, we highlight that when a ff-EBM is a feedforward net (Lemma~\ref{lma:ff}), then the proposed algorithm reduces to BP (Corollary~\ref{cor:ff-bp}).
\end{itemize}

\begin{lemma}[Lagrangian-based approach]
\label{lma:lagrangian}
Assuming a ff-EBM (Def.~\ref{def:ff-EBM-formal}), we denote $s^1_\star, x^1_\star, \cdots, s^{N-1}_\star, \hat{o}_\star$ the states computed during the forward pass as prescribed by Eq.~(\ref{def:ff-EBM-inference-formal}). Then, the gradients of the objective function $\mathcal{C}:=\ell(\hat{o}(s^{N-1}_\star), y)$ as defined in the multilevel optimization problem (Eq.~(\ref{def:multilevel-optim})), where it is assumed that $\ell$ is differentiable, read:

\begin{align}
\left\{
    \begin{array}{l}
    d_{\omega^{N}} \mathcal{C} = \partial_{2}F^N(s_\star^{N-1}, \omega^N)^\top \cdot \partial_1 \ell(\hat{o}_\star, y), \\
    d_{\theta^k} \mathcal{C} = \nabla^2_{1, 2} \widetilde{E}^k(s^k_\star, \theta^k, s^{k-1}_\star, \omega^k) \cdot \lambda_\star^k \quad \forall k = 1 \cdots N-1, \\
    d_{\omega^k} \mathcal{C} = \nabla^2_{1, 4} \widetilde{E}^k(s^k_\star, \theta^k, s^{k-1}_\star, \omega^k) \cdot \lambda_\star^k \quad \forall k = 1 \cdots N-1,
    \end{array}
\right.
\end{align}
where $\lambda^1_\star, \cdots, \lambda^{N-1}_\star$ satisfy the following conditions:
\begin{align}
\left\{
    \begin{array}{l}
    \nabla_{s^{N-1}} \ell(\hat{o}(s^{N-1}_\star), y) + \nabla^2_1 \widetilde{E}^{N-1}(s^{N-1}_\star, \theta^{N-1}, s^{N-2}_\star, \omega^{N-1}) \cdot \lambda_\star^{N-1} = 0 \\
    \forall k =N-2, \cdots, 1: \\
    \quad \nabla^2_{1, 3}\widetilde{E}^{k+1}\left(s^{k+1}_\star, \theta^{k+1}, s^k_\star, \omega^{k+1} \right) \cdot \lambda_\star^{k+1} + \nabla_1^2 \widetilde{E}^k\left(s^k_\star, \theta^k, s^{k-1}_\star, \omega^k\right)\cdot \lambda_\star^k = 0 
    \end{array}
\right.
\end{align}
\end{lemma}

\begin{proof}[Proof of Lemma~\ref{lma:lagrangian}]
Denoting $s := (s^1, \cdots, s^{N-1})^\top$ the state variables of the energy-based blocks, $\lambda := (\lambda^1, \cdots, \lambda^{N-1})^\top$ the Lagrangian multipliers associated with each of these variables, $W := \{\theta_1, \omega_1, \cdots, \theta_{N-1}, \omega_{N-1}\}$ the energy-based and feedforward parameters and $\hat{o}(s^{N-1}) := F^N\left(s^{N-1}, \omega^{N-1}\right)$ the logits, the Lagrangian of the multilevel optimization problem as defined in Eq.~(\ref{def:multilevel-optim}) reads:

\begin{equation}
    \mathcal{L}(s, \lambda, W) := \ell\left(\hat{o}(s^{N-1}), y\right) + \sum_{k=1}^{N-1}\lambda^{k^\top} \cdot \nabla_1 \widetilde{E}^k(s^k, \theta^k, s^{k-1}, \omega^k), \quad s^0 := x
\end{equation}

Writing the associated Karush-Kuhn-Tucker (KKT) conditions $\partial_{1, 2}\mathcal{L}(s_\star, \lambda_\star, W) := 0$ satisfied by optimal states and Lagrangian multipliers $s_\star$ and $\lambda_\star$, we get :

\begin{align}
&\nabla_1 \widetilde{E}^k (s^k_\star, \theta^k, s^{k-1}_\star, \omega^k) = 0 \quad \forall k = 1, \cdots, N- 1\label{eq:kkt:1}\\
&\nabla_{s^{N-1}} \ell(\hat{o}(s^{N-1}_\star), y) + \nabla^2_1 \widetilde{E}^{N-1}(s^{N-1}_\star, \theta^{N-1}, s^{N-2}_\star, \omega^{N-1}) \cdot \lambda_\star^{N-1} = 0  \label{eq:kkt:2} \\
&\nabla^2_{1, 3}\widetilde{E}^{k+1}\left(s^{k+1}_\star, \theta^{k+1}, s^k_\star, \omega^{k+1} \right) \cdot \lambda_\star^{k+1} + \nabla_1^2 \widetilde{E}^k\left(s^k_\star, \theta^k, s^{k-1}_\star, \omega^k\right)\cdot \lambda_\star^k = 0 \quad \forall k =N-2, \cdots, 1
 \label{eq:kkt:3} 
\end{align}

Eq.~(\ref{eq:kkt:1}) governs the bottom-up block-wise relaxation procedure (as depicted in Alg.~\ref{alg:inference-ff-ebm}), while Eq.~(\ref{eq:kkt:2}) and  Eq.~(\ref{eq:kkt:3}) governs error propagation in the last block and previous blocks respectively. Given $s_\star$ and $\lambda_\star$ by Eq.~(\ref{eq:kkt:1}) -- Eq.~(\ref{eq:kkt:3}), the \emph{total} derivative of the loss function with respect to the model parameters read:

\begin{align}
    d_W \ell(\hat{o}_\star, y) &= d_W \left(\ell\left(\hat{o}_\star, y\right) + \sum_{k=1}^{N-1}\lambda^{k^\top}_\star \cdot \underbrace{\nabla_1 \widetilde{E}^k(s^k_\star, \theta^k, s^{k-1}_\star, \omega^k)}_{=0 \quad \mbox{(Eq.~(\ref{eq:kkt:1}))}}\right) \nonumber \\
    &= d_W \mathcal{L}(s_\star, \lambda_\star, W) \nonumber \\
    &= d_W s_\star^\top \cdot \underbrace{\partial_1 \mathcal{L}(s_\star, \lambda_\star, W)}_{=0 \quad \mbox{(Eq.~(\ref{eq:kkt:1}}))} + d_W \lambda_\star^\top \cdot \underbrace{\partial_2 \mathcal{L}(s_\star, \lambda_\star, W)}_{=0 \quad \mbox{(Eq.~(\ref{eq:kkt:2})--(\ref{eq:kkt:3}))}} + \partial_3  \mathcal{L}(s_\star, \lambda_\star, W) \nonumber \\
    &= \partial_3  \mathcal{L}(s_\star, \lambda_\star, W) \label{eq:kkt:gradient}
\end{align}

More precisely, applying Eq.~(\ref{eq:kkt:gradient}) to the feedforward and energy-based block parameters yields:

\begin{align*}
    d_{\omega^{N}} \ell(\hat{o}_\star, y) &= \partial_{2}F^N(s_\star^{N-1}, \omega^N)^\top \cdot \nabla_1 \ell(\hat{o}_\star, y), \\
    d_{\theta^k} \ell(\hat{o}_\star, y) &= \nabla^2_{1, 2} \widetilde{E}^k(s^k_\star, \theta^k, s^{k-1}_\star, \omega^k) \cdot \lambda_\star^k \quad \forall k = 1 \cdots N-1 \\
    d_{\omega^k} \ell(\hat{o}_\star, y) &= \nabla^2_{1, 4} \widetilde{E}^k(s^k_\star, \theta^k, s^{k-1}_\star, \omega^k) \cdot \lambda_\star^k \quad \forall k = 1 \cdots N-1
\end{align*}
    
\end{proof}

\begin{lemma}[Computing Lagrangian multipliers by EP]
\label{lma:lagrangian-ep}
Under the same hypothesis as Lemma~\ref{lma:lagrangian}, we define the nudged state of block $k$, denoted as $s^k_\beta$,  implicitly through $\nabla_1 \mathcal{F}^k(s^k_\beta, \theta^k, x^k_\star, \delta s^k, \beta) = 0$ with:
\begin{equation}
\mathcal{F}^k(s^k, \theta^k, x^k_\star, \delta s^k, \beta) := E^k(s^k, \theta^k, x^k_\star) + \beta s^{k^\top}\cdot \delta s^k.
\end{equation}

Defining $(\delta s^k)_{k=1 \cdots N-1}$ recursively as:

\begin{equation}
    \delta s^{N-1} := \nabla_{s^{N-1}} \ell(\hat{o}_\star, y), \quad \delta s^k := d_\beta \left.\left(\nabla_3 \widetilde{E}^{k+1}\left(s^{k+1}_\beta, \theta^{k+1}, s^k_\star, \omega^{k+1} \right)\right)\right|_{\beta=0} \ \forall k=1\cdots N-2,
    \label{lma:delta-s}
\end{equation}

then we have:

\begin{equation}
    \lambda_\star^k = d_\beta \left(s_\beta^k \right)|_{\beta=0} \quad \forall k = 1 \cdots N - 1,
\end{equation}

where $(\lambda_k)_{k=1 \cdots N-1}$ are the Lagrangian multipliers associated to the multilevel optimization problem defined in Eq.~(\ref{def:multilevel-optim}).
\end{lemma}

\begin{proof}[Proof of Lemma~\ref{lma:lagrangian-ep}] We prove this result by backward induction on $k$. \\

{\bfseries Initialization ($k = N-1$)}. By definition, $s^{N-1}_\beta$ satisfies :

\begin{equation}
    \beta \nabla_{s^{N-1}} \ell\left(\hat{o}_\star, y\right) + \nabla_1 \widetilde{E}^{N-1}\left(s^{N-1}_\beta, \theta^{N-1}, s^{N-2}_\star, \omega^{N-1}\right) = 0
    \label{eq:lagrangian-ep:step:1}
\end{equation}

Differentiating Eq.~(\ref{eq:lagrangian-ep:step:1}) with respect to $\beta$ and evaluating the resulting expression at $\beta = 0$, we obtain:

\begin{equation}
    \nabla_{s^{N-1}} \ell\left(\hat{o}_\star, y\right) + \nabla_1^2 \widetilde{E}^{N-1}\left(s^{N-1}_\star, \theta^{N-1}, s^{N-2}_\star, \omega^{N-1}\right)\cdot d_\beta s^{N-1}_\beta|_{\beta=0} = 0
    \label{eq:lagrangian-ep:step:2}
\end{equation}

Substracting out Eq.~(\ref{eq:kkt:2}) defining the Lagrangian multiplier $\lambda^{N-1}_\star$ and Eq.~(\ref{eq:lagrangian-ep:step:2}), we obtain:

\begin{equation}
    \nabla_1^2 \widetilde{E}^{N-1}\left(s^{N-1}_\star, \theta^{N-1}, s^{N-2}_\star, \omega^{N-1}\right)\cdot \left( d_\beta s^{N-1}_\beta|_{\beta=0} - \lambda_\star^{N-1}\right) = 0
\end{equation}

By invertibility of $\nabla_1^2 \widetilde{E}^{N-1}\left(s^{N-1}_\star, \theta^{N-1}, s^{N-2}_\star, \omega^{N-1}\right)$, we therefore have that:

\begin{equation}
    \lambda_\star^{N-1}=d_\beta s^{N-1}_\beta|_{\beta=0}
\end{equation}

{\bfseries Backward induction step ($k + 1 \to k$).} Let us assume that $\lambda^{k+1}_\star = d_\beta s^{k+1}_\beta |_{\beta=0}$. We want to prove that $\lambda^{k}_\star = d_\beta s^{k}_\beta |_{\beta=0}$. Again, $s^{k+1}_\beta$ satisfies by definition:

\begin{equation}
    \beta \delta s^k + \nabla_1 \widetilde{E}^k\left(s^k_\beta, \theta^k, s^{k-1}_\star, \omega^k\right) = 0, \quad \delta s^k := d_\beta \left.\left(\nabla_3 \widetilde{E}^{k+1}\left(s^{k+1}_\beta, \theta^{k+1}, s^k_\star, \omega^{k+1} \right)\right)\right|_{\beta=0}.
    \label{eq:lagrangian-ep:step:3}
\end{equation}

On the one hand, proceeding as for the initialization step, differentiating Eq.~(\ref{eq:lagrangian-ep:step:3}) with respect to $\beta$ and taking $\beta=0$ yields:

\begin{equation}
    \delta s^k + \nabla_1^2 \widetilde{E}^k(s^k_\star, \theta^k, s^{k-1}_\star, \omega^k)\cdot d_\beta s^k_\beta |_{\beta=0} = 0.
    \label{eq:lagrangian-ep:step:4}
\end{equation}

On the other hand, note that $\delta s^k$ rewrites :

\begin{align}
    \delta s^k &= d_\beta \left.\left(\nabla_3 \widetilde{E}^{k+1}\left(s^{k+1}_\beta, \theta^{k+1}, s^k_\star, \omega^{k+1} \right)\right)\right|_{\beta=0} \nonumber\\
    &=\nabla_{1, 3}^2 \widetilde{E}^{k+1}\left(s^{k+1}_\star, \theta^{k+1}, s^k_\star, \omega^{k+1}\right) \cdot \left.d s^{k+1}_\beta \right|_{\beta=0} \nonumber \\
    &= \nabla_{1, 3}^2 \widetilde{E}^{k+1}\left(s^{k+1}_\star, \theta^{k+1}, s^k_\star, \omega^{k+1}\right) \cdot \lambda^{k+1}_\star,
    \label{eq:lagrangian-ep:step:5}
\end{align}

where we used at the last step the recursion hypothesis. Therefore combining Eq.~(\ref{eq:lagrangian-ep:step:4}) and Eq.~(\ref{eq:lagrangian-ep:step:5}), we get:

\begin{equation}
    \nabla_{1, 3}^2 \widetilde{E}^{k+1}\left(s^{k+1}_\star, \theta^{k+1}, s^k_\star, \omega^{k+1}\right) \cdot \lambda^{k+1}_\star + \nabla_1^2 \widetilde{E}^k(s^k_\star, \theta^k, s^{k-1}_\star, \omega^k)\cdot d_\beta s^k_\beta |_{\beta=0} = 0.
    \label{eq:lagrangian-ep:step:6}
\end{equation}

Finally, we substract out Eq.~(\ref{eq:kkt:3}) and Eq.~(\ref{eq:lagrangian-ep:step:6}) to obtain:

\begin{equation}
\nabla_1^2 \widetilde{E}^k(s^k_\star, \theta^k, s^{k-1}_\star, \omega^k)\cdot \left( d_\beta s^k_\beta |_{\beta=0}  - \lambda_\star^k\right)= 0.
\label{eq:lagrangian-ep:step:7}
\end{equation}

We conclude again by invertibility of $\nabla_1^2 \widetilde{E}^k(s^k_\star, \theta^k, s^{k-1}_\star, \omega^k)$ that $\lambda^k_\star = d_\beta s^k_\beta |_{\beta=0}$.

\end{proof}

\begin{theorem}[Formal]
\label{theorem:main-result-formal}
Assuming a model of the form Eq.~(\ref{def:deeply-nested-model}), we denote $s^1_\star, x^1_\star, \cdots, s^{N-1}_\star, \hat{o}_\star$ the states computed during the forward pass as prescribed by Alg.~\ref{alg:inference-ff-ebm}. 
We define the nudged state of block $k$, denoted as $s^k_\beta$,  implicitly through $\nabla_1 \mathcal{F}^k(s^k_\beta, \theta^k, x^k_\star, \delta s^k, \beta) = 0$ with:
\begin{equation}
\mathcal{F}^k(s^k, \theta^k, x^k_\star, \delta s^k, \beta) := E^k(s^k, \theta^k, x^k_\star) + \beta s^{k^\top}\cdot \delta s^k.
\end{equation}

Denoting $\delta s^k$ and $\Delta x^k$ the error signals computed at the input of the feedforward block $F^k$ and of the energy-based block $E^k$ respectively, $g_{\theta^k}$ and $g_{\omega^k}$ the gradients of the loss function:

\begin{equation}
    \forall k = 1, \cdots, N-1: \ g_{\theta^k} := d_{\theta^k} \mathcal{C}, \qquad \forall k =1 \cdots N: \ g_{\omega^k} := d_{\omega^k} \mathcal{C},
\end{equation}

then the following chain rule applies:

\begin{align}
&\delta s^{N-1} := \nabla_{s^{N-1}}\ell(\hat{o}_\star, y), \quad g_{\omega^N} = \partial_2 F^{N}\left(s^{N-1}_\star, \omega^N \right)^\top \cdot \nabla_{1} \ell(\hat{o}_\star, y) \label{lma:main-result-formal-initial-error-signal}\\
&\forall k=1 \cdots N-1: \nonumber\\
&\left\{
    \begin{array}{l}
\Delta x^k = \left.d_{\beta}\left(\nabla_{3}E^k (s^k_\beta, \theta^k, x^k_\star)\right)\right|_{\beta=0}, \quad g_{\theta^k} = \left.d_{\beta}\left(\nabla_{2}E^k (s^k_\beta, \theta^k, x^k_\star)\right)\right|_{\beta=0} \\
    \delta s^{k-1} = \partial_1 F^k\left(s^{k-1}_\star, \omega^k \right)^\top \cdot \Delta x^k, \quad g_{\omega^k} = \partial_2 F^k\left(s^{k-1}_\star, \omega^k\right)^\top \cdot \Delta x^k
    \end{array}
\right.
\end{align}
\end{theorem}

\begin{proof}[Proof of Theorem~\ref{theorem:main-result-formal}]

Combining Lemma~\ref{lma:lagrangian} and Lemma~\ref{lma:lagrangian-ep}, the following chain rule computes loss gradients correctly:

\begin{align}
&\delta s^{N-1} := \nabla_{s^{N-1}}\ell(\hat{o}_\star, y), \quad g_{\omega^N} = \partial_2 F^{N}\left(s^{N-1}_\star, \omega^N \right)^\top \cdot \nabla_{1} \ell(\hat{o}_\star, y)\\
&\forall k=1 \cdots N-1: \nonumber\\
&\left\{
    \begin{array}{l}
    \Delta s^{k-1} = d_\beta \left.\left(\nabla_3 \widetilde{E}^{k}\left(s^{k}_\beta, \theta^{k}, s^{k-1}_\star, \omega^{k} \right)\right)\right|_{\beta=0}, \quad g_{\theta^k} = \nabla^2_{1, 2} \widetilde{E}^k(s^k_\star, \theta^k, s^{k-1}_\star, \omega^k) \cdot d_\beta s^k_\beta |_{\beta=0}\\
    g_{\omega^k} = \nabla_{1, 4}^2 \widetilde{E}^k(s^k_\star, \theta^k, s^{k-1_\star, \omega^k})\cdot d_\beta s^k_\beta |_{\beta=0}
    \end{array}
\right.
\end{align}

Therefore to conclude the proof, we need to show that $\forall k = 1, \cdots, N-1$:

\begin{align}
d_\beta \left.\left(\nabla_3 \widetilde{E}^{k}\left(s^{k}_\beta, \theta^{k}, s^{k-1}_\star, \omega^{k} \right)\right)\right|_{\beta=0} &=  \partial_1 F^k\left(s^{k-1}_\star, \omega^k \right)^\top \cdot \left.d_{\beta}\left(\nabla_{3}E^k (s^k_\beta, \theta^k, x^k_\star)\right)\right|_{\beta=0}  \label{eq:main-result-formal:1} \\
\nabla^2_{1, 2}\widetilde{E}^k(s^k_\star, \theta^k, s^{k-1}_\star, \omega^k) \cdot d_\beta s^k_\beta |_{\beta=0} &=\left.d_{\beta}\left(\nabla_{2}E^k (s^k_\beta, \theta^k, x^k_\star)\right)\right|_{\beta=0} \label{eq:main-result-formal:2} \\
\nabla_{1, 4}^2 \widetilde{E}^k(s^k_\star, \theta^k, s^{k-1}_\star, \omega^k)\cdot d_\beta s^k_\beta |_{\beta=0} &= \partial_2 F^k\left(s^{k-1}_\star, \omega^k\right)^\top \cdot \left.d_{\beta}\left(\nabla_{3}E^k (s^k_\beta, \theta^k, x^k_\star)\right)\right|_{\beta=0}\label{eq:main-result-formal:3}
\end{align}

Let $k \in [1, N-1]$. We prove Eq.~(\ref{eq:main-result-formal:1}) as:

\begin{align*}
    d_\beta \left.\left(\nabla_3 \widetilde{E}^{k}\left(s^{k}_\beta, \theta^{k}, s^{k-1}_\star, \omega^{k} \right)\right)\right|_{\beta=0} 
    &=d_\beta \left.\left(\nabla_{s^{k-1}} E^{k}\left(s^{k}_\beta, \theta^{k}, F^{k}\left(s^{k-1}_\star, \omega^{k} \right)\right)\right)\right|_{\beta=0} \nonumber\\
    &= \partial_1 F^{k}\left(s^{k-1}_\star, \omega^{k} \right)^\top \cdot \left.d_{\beta}\left(\nabla_{3}E^{k} (s^{k}_\beta, \theta^{k}, x^{k}_\star)\right)\right|_{\beta=0}
\end{align*}

Eq.~(\ref{eq:main-result-formal:2}) can be obtained as:

\begin{align*}
    \nabla^2_{1, 2} \widetilde{E}^k(s^k_\star, \theta^k, s^{k-1}_\star, \omega^k) \cdot d_\beta s^k_\beta |_{\beta=0}
    &= \left.d_\beta \left(\nabla_2 \widetilde{E}^k(s^k_\beta, \theta^k, s^{k-1}_\star, \omega^k) \right)\right|_{\beta=0} \nonumber \\
    &= \left.d_{\beta}\left(\nabla_{2}E^k (s^k_\beta, \theta^k, x^k_\star)\right)\right|_{\beta=0}
\end{align*}

Finally and similarly, we have:

\begin{align*}
    \nabla_{1, 4}^2 \widetilde{E}^k(s^k_\star, \theta^k, s^{k-1}_\star, \omega^k)\cdot d_\beta s^k_\beta |_{\beta=0} &= \left. d_\beta \left( \nabla_{4} \widetilde{E}^k(s^k_\beta, \theta^k, s^{k-1}_\star, \omega^k)\right) \right|_{\beta=0} \\
    &=  \left. d_\beta \left( \nabla_{\omega^k} E^k(s^k_\beta, \theta^k, F^k\left(s^{k-1}_\star, \omega^k\right))\right) \right|_{\beta=0} \\
    &= \left. d_\beta \left(\partial_2 F\left(s^{k-1}_\star, \omega^k \right)^\top \cdot \nabla_{3} E^k(s^k_\beta, \theta^k, x^k_\star)\right) \right|_{\beta=0} \\
    &= \partial_2 F\left(s^{k-1}_\star, \omega^k \right)^\top \cdot \left. d_\beta \left(\nabla_{3} E^k(s^k_\beta, \theta^k, x^k_\star)\right) \right|_{\beta=0} 
\end{align*}
\end{proof}

\begin{corollary}
\label{cor:ff-bp}
Under the same hypothesis as Theorem~\ref{theorem:main-result-formal} and Lemma~\ref{lma:ff}, then the following chain rule applies to compute error signals backward from the output layer:
\begin{align}
\left\{
    \begin{array}{l}
    \delta s^{N-1} := \nabla_{s^{N-1}}\ell(\hat{o}_\star, y), \quad g_{\omega^N} = \nabla_{\omega^N} \ell(\hat{o}_\star, y) \\
    \Delta x^k = \sigma'(x^k_\star) \odot \delta s^k\\
    \delta s^{k-1} = \partial_1 F^k\left(s^{k-1}_\star, \omega^k \right)^\top \cdot \Delta x^k, \quad g_{\omega^k} = \partial_2 F^k\left(s^{k-1}_\star, \omega^k\right)^\top \cdot \Delta x^k
    \end{array}
\right.
\label{eq:ff-bp:chain-rule}
\end{align}
\end{corollary}

\begin{proof}[Proof of Corollary~\ref{cor:ff-bp}]
Let $k \in [1, N-1]$. As we can directly apply Theorem~\ref{theorem:main-result-formal} here, proving the result simply boils down to showing that:

\begin{equation}
    \Delta x^k = \sigma'(x^k_\star) \odot \delta s^k
    \label{eq:ff-bp:0}
\end{equation}

First, we notice that when $E^k$ is of the form of Eq.~(\ref{ff:energy-block-k-formal}), then $\Delta x^k$ reads as:

\begin{equation}
    \Delta x^k = \left.d_{\beta}\left(\nabla_{3}E^k (s^k_\beta, \theta^k, x^k_\star)\right)\right|_{\beta=0} = -\left. d_\beta \left(s^k_\beta\right) \right|_{\beta=0}.
    \label{eq:ff-bp:1}
\end{equation}

$s^k_\beta$ satisfies, by definition and when $U^k = 0$:

\begin{align}
    &\sigma^{-1}(s^k_\beta) - x^k_\star + \beta \delta s^k = 0 \nonumber \\
    \Leftrightarrow \quad & s^k_\beta = \sigma\left( x^k_\star - \beta \delta s^k\right)\label{eq:ff-bp:2}
\end{align}

Combining Eq.~(\ref{eq:ff-bp:1}) and Eq.~(\ref{eq:ff-bp:2})
yields Eq.~(\ref{eq:ff-bp:0}), and therefore, along with Theorem~\ref{theorem:main-result-formal}, the chain-rule Eq.~(\ref{eq:ff-bp:chain-rule}).
\end{proof}

We showcase in Alg.~\ref{alg:implicit-bp-ep-chain-ff} the resulting algorithm with \emph{finite} $\beta$ and \emph{implicit} BP-EP chaining, with lines in blue highlighting differences with the general algorithm Alg.~\ref{alg:implicit-bp-ep-chain}. 

\begin{algorithm}[H]
    \caption{Implicit BP-EP gradient chaining with $U^k=0$}
    \label{alg:implicit-bp-ep-chain-ff}
    \begin{algorithmic}[1]
    \State $\delta s, g_{\omega^N} \gets \nabla_{s^{N-1}}\ell(\hat{o}_\star, y), \nabla_{\omega^N} \ell(\hat{o}_\star, y) $ \Comment{Single backprop step}
    \For{$k=N - 1 \cdots 1$}
        \State \textcolor{blue}{$s_\beta, \ s_{-\beta} \gets \sigma \left(x^k_\star - \beta \delta s^k\right), \sigma \left(x^k_\star + \beta \delta s^k\right)$} \Comment{EP through $\widetilde{E}^k$}
        \State $g_{\omega^k} \gets \frac{1}{2\beta}\left(\nabla_{4} \widetilde{E}^k(s_\beta, \theta^k, s^{k-1}_\star, \omega^k) - \nabla_{4} \widetilde{E}^k(s_{-\beta}, \theta^k, s^{k-1}_\star, \omega^k)\right)$ \Comment{i-BP through $F^k$}
    \State $ \delta s \gets \frac{1}{2\beta}\left(\nabla_{3} \widetilde{E}^k(s_\beta, \theta^k, s^{k-1}_\star, \omega^k) - \nabla_{3} \widetilde{E}^k(s_{-\beta}, \theta^k, s^{k-1}_\star, \omega^k)\right)$
    \EndFor
    \end{algorithmic}
\end{algorithm}

\subsection{Explicit BP-EP chaining}
\label{app:explicit-bp-ep-chaining}

We presented in Alg.~\ref{alg:implicit-bp-ep-chain} a ``pure'' EP algorithm where the BP-EP gradient chaining is \emph{implicit}. We show below, inside Alg.~\ref{alg:explicit-bp-ep-chain}, an alternative implementation (equivalent in the limit $\beta \to 0$) where the use of BP through feedforward modules is \emph{explicit} and which is the direct implementation of Theorem~\ref{theorem:main-result-formal}. We also show the resulting algorithm when the ff-EBM reduces to a feedforward net (Lemma~\ref{lma:ff}) inside Alg.~\ref{alg:explicit-bp-ep-chain-ff}, highlight in blue the statements which differ from the general case presented inside Alg.~\ref{alg:explicit-bp-ep-chain}.

\begin{algorithm}[H]
    \caption{Explicit BP-EP gradient chaining (Theorem~(\ref{theorem:main-result}))}
    \label{alg:explicit-bp-ep-chain}
    \begin{algorithmic}[1]
    \State $\delta s, g_{\omega^N} \gets \nabla_{s^{N-1}}\ell(\hat{o}_\star, y), \nabla_{\omega^N} \ell(\hat{o}_\star, y) $ \Comment{Single backprop step}
    \For{$k=N - 1 \cdots 1$}
        \State $s_\beta \gets \underset{s}{\mathbf{Optim}}\left[E^k(s, \theta^k, x^k_\star) + \beta s^\top \cdot \delta s\right]$ \Comment{EP through $E^k$}
        \State $s_{-\beta} \gets \underset{s}{\mathbf{Optim}}\left[E^k(s, \theta^k, x^k_\star) - \beta s^\top \cdot \delta s\right]$
        \State $ g_{\theta^k} \gets \frac{1}{2\beta}\left(\nabla_{2} E^k(s_\beta, \theta^k, x^k_\star) - \nabla_{2} E^k(s_{-\beta}, \theta^k, x^k_\star)\right)$
        \State $\Delta x \gets \frac{1}{2\beta}\left(\nabla_{3} E^k(s_\beta, \theta^k, x^k_\star) - \nabla_{3} E^k(s_{-\beta}, \theta^k, x^k_\star)\right)$
        \State $g_{\omega^k} \gets \partial_2 F^k\left(s^{k-1}_\star, \omega^k\right)^\top \cdot \Delta x$ \Comment{Explicit BP through $F^k$}
        \State $\delta s \gets  \partial_1 F^k\left(s^{k-1}_\star, \omega^k \right)^\top \cdot \Delta x$
    \EndFor
    \end{algorithmic}
\end{algorithm}

\begin{algorithm}[H]
    \caption{Explicit BP-EP gradient chaining with $U^k=0$}
    \label{alg:explicit-bp-ep-chain-ff}
    \begin{algorithmic}[1]
    \State $\delta s, g_{\omega^N} \gets \nabla_{s^{N-1}}\ell(\hat{o}_\star, y), \nabla_{\omega^N} \ell(\hat{o}_\star, y) $ \Comment{Single backprop step}
    \For{$k=N - 1 \cdots 1$} 
        \State $\textcolor{blue}{\Delta x \gets -\frac{1}{2\beta}\left(\sigma \left(x^k_\star - \beta \delta s^k\right) - \sigma \left(x^k_\star + \beta \delta s^k\right)\right)}$
        \State $g_{\omega^k} \gets \partial_2 F^k\left(s^{k-1}_\star, \omega^k\right)^\top \cdot \Delta x$ \Comment{Explicit BP through $F^k$}
        \State $\delta s \gets  \partial_1 F^k\left(s^{k-1}_\star, \omega^k \right)^\top \cdot \Delta x$
    \EndFor
    \end{algorithmic}
\end{algorithm}

\newpage

\subsection{Model and algorithm details}
\label{subsec:details-model}

\paragraph{Equilibrium computation.} As mentioned in Section~\ref{subsec:ff-ebm}, the energy function of the k$^{\rm th}$  EB block has the form:

\begin{equation}
    E^k(s^k, \theta^k, x^k) := G^k(s^k) - s^{k^\top} \cdot x^k + U^k(s^k, \theta^k)
    \label{app:def:energy-block-k-version-1},
\end{equation}

where $x^k$ is the output of the preceding feedforward block. For a given choice of a continuously invertible activation function, $G^k_\sigma$ is defined as:

\begin{equation}
G^k_\sigma(s^k) :=\sum_{i=1}^{\dim(s^k)}\int^{s_i}\sigma_i^{-1}(u_i)du_i \quad \mbox{such that} \quad \nabla G^k_\sigma(s^k)_{i} = \sigma^{-1}_i(s^k_i) \quad \forall i = 1 \cdots \dim(s^k).
\end{equation}

To be more explicit and as we did previously, we re-write the augmented energy-function which encompasses both the k$^{\rm th}$ EB block and the feedforward module that precedes it:

\begin{equation}
    \widetilde{E}^k(s^k, \theta^k, s^{k-1}_\star, \omega^k) := E^k\left(s^k, \theta^k, F^k\left(s^{k-1}_\star, \omega^{k} \right)\right)
    \label{app:def:energy-block-k-version-2}.
\end{equation}

We showed that when is chosen such that $\nabla G = \sigma^{-1}$ for some activation function $\sigma$, then the steady state of the k$^{\rm th}$ block reads:

\begin{equation}
    s^k_\star := \sigma\left(x^k - \nabla_1 U^k(s^k_\star, \theta^k) \right),
    \label{app:eq:steady-state-block-k}
\end{equation}

which justifies the following fixed-point iteration scheme, when the block is influenced by some error signal $\delta s$ with nudging strength $\beta$:

\begin{equation}
    s^k_{\pm\beta, t + 1} \gets \sigma\left(x^k - \nabla_1 U^k(s^k_{\pm\beta, t}, \theta^k) \mp \beta \delta s^k \right).
    \label{app:eq:fixed-point-iteration}
\end{equation}

The dynamics prescribed by Eq.~\ref{app:eq:fixed-point-iteration} are also used for the inference phase with $\beta=0$. To further refine Eq.~(\ref{app:eq:fixed-point-iteration}), let us re-write Eq.~(\ref{app:eq:fixed-point-iteration}) with a layer index $\ell$ where $\ell \in [1, L_k]$ with $L_k$ being the number of layers in the k$^{\rm th}$ block, and replacing $x^k$ by its explicit expression:

\begin{equation}
\forall \ell = 1 \cdots L_k: \    s^k_{\ell, \pm\beta, t + 1} \gets \sigma\left(F^k\left(s^{k-1}_\star, \omega^{k-1}\right) - \nabla_{s^k_\ell} U^k(s^k_{\pm\beta, t}, \theta^k) \mp \beta \delta s^k \right).
    \label{app:eq:fixed-point-iteration-2}
\end{equation}

As done in past EP works \citep{ernoult2019updates, laborieux2021scaling, laborieux2022holomorphic, laborieux2023improving, scellier2024energy} and for notational convenience, we introduce the \emph{primitive function} of the k$^{\rm th}$ block as:

\begin{equation}
    \Phi^k\left(s^k, \theta^k, s^{k-1}_\star, \omega^k\right) := s^{k^\top}\cdot F^k\left(s^{k-1}_\star, \omega^{k}\right) - U^k(s^k, \theta^k)
    \label{app:def:primitive}
\end{equation}

such that Eq.~(\ref{app:eq:fixed-point-iteration-2}) re-writes:

\begin{equation}
  \forall \ell = 1 \cdots L_k:  s^k_{\ell, \pm\beta, t + 1} \gets \sigma\left(\nabla_{s^k_\ell} \Phi\left(s^k_{\pm \beta, t}, \theta^k, s^{k-1}_\star, \omega^k\right) \mp \beta \delta s^k \right).
    \label{app:eq:fixed-point-iteration-3}
\end{equation}

Eq.~(\ref{app:eq:fixed-point-iteration-3}) depicts a \emph{synchronous} scheme where all layers are simultaneously updated at each timestep. Another possible scheme, employed by \cite{scellier2024energy}, instead prescribes to \emph{asynchronously} update odd and even layers and was shown to speed up convergence in practice:

\begin{align}
\left\{
    \begin{array}{ll}
  \forall \ \mbox{odd} \ \ell \in \{1, \cdots, L_k\}:  &\ s^k_{\ell, \pm\beta, t + \frac{1}{2}} \gets \sigma\left(\nabla_{s^k_\ell} \Phi\left(s^k_{\pm\beta, t}, \theta^k, s^{k-1}_\star, \omega^k\right) \mp \beta \delta s^k \right),\\
  \forall \ \mbox{even} \ \ell \in \{1, \cdots, L_k\}:  & \ s^k_{\ell, \pm\beta, t + 1} \gets \sigma\left(\nabla_{s^k_\ell} \Phi\left(s^k_{\pm\beta, t + \frac{1}{2}}, \theta^k, s^{k-1}_\star, \omega^k\right) \mp \beta \delta s^k \right).
    \end{array}
\right.
\end{align}

We formally depict this procedure as the subroutine \texttt{Asynchronous} inside Alg.~\ref{alg:asynchronous}. In practice, we observe that it was more practical to use a \emph{fixed} number of iterations rather than using a convergence criterion with a fixed threshold.

\begin{algorithm}[ht!]{
\emph{Input}: $T$, $\theta^k$, $\omega^k$, $s^{k-1}_\star$, $\beta$, $\delta s^k$ \\
\emph{Output}: $s^k_\beta$
}
    \caption{\texttt{Asynchronous} (for all blocks {\bfseries until penultimate})}
    \label{alg:asynchronous}
    \begin{algorithmic}[1]
    \State $s^k \gets 0 $
    \For{$t=1\cdots T$}
        \State $\forall \ \mbox{odd} \ \ell \in \{1, \cdots, L_k\}:  \ s^k_{\ell, \beta} \gets \sigma\left(\nabla_{s_\ell^k} \Phi\left(s^k_{\beta}, \theta^k, s^{k-1}_\star, \omega^k\right) - \beta \delta s^k \right)$
        \State $\forall \ \mbox{even} \ \ell \in \{1, \cdots, L_k\}:  \ s^k_{\ell, \beta} \gets \sigma\left(\nabla_{s^k_\ell} \Phi\left(s^k_{\beta}, \theta^k, s^{k-1}_\star, \omega^k\right) - \beta \delta s^k \right)$
    \EndFor
    \end{algorithmic}
\end{algorithm}

\paragraph{Nudging the last block.} From looking at the procedure prescribed by Theorem~\ref{theorem:main-result} and algorithms thereof (Alg.~\ref{alg:implicit-bp-ep-chain}, Alg.~\ref{alg:explicit-bp-ep-chain}), all the error signals used to nudge the EB blocks are \emph{stationary}, including the top-most block where the loss error signal is fed in. Namely, the augmented energy function of the last block reads as:

\begin{equation}
\mathcal{F}^{N-1}(s^{N-1}, \theta^{N-1}, x^{N-1}_\star, \beta) := E^{N-1}(s^{N-1}, \theta^{N-1}, x^{N-1}_\star) + \beta s^{{N-1}^\top}\cdot \nabla_{s^{N-1}} \ell(\hat{o}_\star, y),
\label{eq:augmented-E-last-block-linearized}
\end{equation}

where $\hat{o}_\star := F^N\left(s^{N-1}_\star, \omega^{N}\right)$ is \emph{constant}. Up to a constant, Eq.~(\ref{eq:augmented-E-last-block}) uses the cost function \emph{linearized around} $s^{N-1}_\star$ instead of the cost function itself. This is, however, in contrast with most EP implementations where the nudging force acting upon the EB block is usually \emph{elastic}, i.e. the nudging depends on the current state of the EB block. More precisely, instead of using Eq.~(\ref{eq:augmented-E-last-block-linearized}), we instead use:

\begin{equation}
\mathcal{F}^{N-1}(s^{N-1}, \theta^{N-1}, x^{N-1}_\star, \beta) := E^{N-1}(s^{N-1}, \theta^{N-1}, x^{N-1}_\star) + \beta \ell(F^{N}(s^{N-1}, \omega^N), y),
\label{eq:augmented-E-last-block}
\end{equation}

This results in the following asynchronous fixed-point dynamics for the last block:

\begin{align*}
\left\{
    \begin{array}{ll}
  \forall \ \mbox{odd} \ \ell \in \{1, \cdots, L_k\}:  &\ s^k_{\ell, \pm\beta, t + \frac{1}{2}} \gets \sigma\left(\nabla_{s^k_\ell} \Phi\left(s^k_{\pm\beta, t}, \theta^k, s^{k-1}_\star, \omega^k\right) \mp \beta \nabla_{s^k} \ell(s^k_{\pm \beta, t}, y) \right),\\
  \forall \ \mbox{even} \ \ell \in \{1, \cdots, L_k\}:  & \ s^k_{\ell, \pm\beta, t + 1} \gets \sigma\left(\nabla_{s^k_\ell} \Phi\left(s^k_{\pm\beta, t + \frac{1}{2}}, \theta^k, s^{k-1}_\star, \omega^k\right) \mp \beta \nabla_{s^k} \ell(s^k_{\pm \beta, t}, y) \right).
    \end{array}
\right.
\end{align*}

The resulting \texttt{Asynchronous} subroutine, applying for the last block, is depicted inside Alg.~\ref{alg:asynchronous-last-block}.

\begin{algorithm}[ht!]{
\emph{Input}: $T$, $\theta^{N-1}$, $\omega^{N-1}$, $\omega^N$, $s^{k-1}_\star$, $\beta$, $\ell$ (cost function), $y$ \\
\emph{Output}: $s^{N-1}_\beta$
}
    \caption{\texttt{Asynchronous} (for {\bfseries last} block)}
    \label{alg:asynchronous-last-block}
    \begin{algorithmic}[1]
    \State $s^{N-1} \gets 0 $
    \For{$t=1\cdots T$}
        \State $\forall \ \mbox{odd} \ \ell \in \{1, \cdots, L_N\}$: 
        \State $\quad s^{N-1}_{\ell, \beta} \gets \sigma\left(\nabla_{s_\ell^{N-1}} \Phi\left(s^{N-1}_{\beta}, \theta^{N-1}, s^{N-2}_\star, \omega^{N-1}\right) - \beta \nabla_{s^{N-1}_\ell} \ell(F^N(s^{N-1}, \omega^N), y) \right)$
        \State $\forall \ \mbox{even} \ \ell \in \{1, \cdots, L_N\}$: 
        \State $\quad s^{N-1}_{\ell, \beta} \gets \sigma\left(\nabla_{s_\ell^{N-1}} \Phi\left(s^{N-1}_{\beta}, \theta^{N-1}, s^{N-2}_\star, \omega^{N-1}\right) - \beta \nabla_{s^{N-1}_\ell} \ell(F^N(s^{N-1}, \omega^N), y) \right)$
    \EndFor
    \end{algorithmic}
\end{algorithm}

\paragraph{Readout.} \cite{laborieux2021scaling} introduced the idea of the ``readout'' whereby the last linear layer computing the loss logits is \emph{not} part of the EB free block dynamics but simply ``reads out'' the state of the penultimate block. In all our experiments we use such a readout in combination with the cross entropy loss function. Using our formalism, our readout is simply the last feedforward transformation used inside $\ell$, namely $F^{N}(\cdot, \omega^N)$.

\paragraph{Deep Hopfield Networks (DHNs).} In our experiments, we used weight matrices of the form:

\begin{equation}
\theta^k = \begin{bmatrix}
        0 & \theta_1^{k^\top} & 0 & & \\
        \theta_1^k & 0 & \theta_2^{k^\top} & & \\
        0 & \theta_2^k & \ddots & \ddots & \\
          &          & \ddots & 0 & \theta_L^{k^\top} \\
         & & & \theta_L^k & 0   
    \end{bmatrix},
\end{equation}

whereby each layer $\ell$ is only connected to its adjacent neighbors. Therefore, fully connected and convolutional DHNs with $L$ layers have an energy function of the form:

\begin{align}
    U^k_{\rm FC}(s^k, \theta^k) &:= -\frac{1}{2}s^{k\top} \cdot \theta^k \cdot s^k = -\frac{1}{2} \sum_{\ell} s^{k^\top}_{\ell + 1} \cdot \theta_\ell^k \cdot s^k_{\ell} \\
    U^k_{\rm CONV}(s^k, \theta^k) &:= -\frac{1}{2}s^k \bullet \left(\theta^k \star s^k\right)  = -\frac{1}{2} \sum_{\ell} s^{k}_{\ell + 1} \bullet \left( \theta_\ell^k \star s^k_{\ell}\right)
\end{align}

\paragraph{Detailed inference algorithm.} With the aforementioned details in hand, we re-write the inference algorithm Alg.~\ref{alg:inference-ff-ebm} presented in the main as a \texttt{Forward} subroutine.

\begin{algorithm}[H]{
\emph{Input:} $T$, $x$, $W = \{\theta^1, \omega^1, \cdots \omega^N \}$ \\
\emph{Output}: $s^1, \cdots, s^{N-1}$ or $\hat{o}$ depending on the context
}
    \caption{\texttt{Forward}}
    \label{alg:detailed-inference-ff-ebm}
    \begin{algorithmic}[1]
    \State $s^0 \gets x$
    \For{$k=1 \cdots N - 1$}
        \State $s^k \gets \texttt{Asynchronous}\left(T, \theta^k, \omega^k, s^{k-1}\right)$\Comment{Alg.~\ref{alg:asynchronous}}
    \EndFor
    \State $\hat{o} \gets F^{N}\left(s, \omega^N\right)$
    \end{algorithmic}
\end{algorithm}

\paragraph{Detailed implicit EP-BP chaining algorithm.}
We provide a detailed implementation of our algorithm presented in the main (Alg.~\ref{alg:implicit-bp-ep-chain}) in Alg.~\ref{alg:implicit-bp-ep-chain-detailed}. As usually done for EP experiments, we always perform a ``free phase'' to initialize the block states (\texttt{Forward} subroutine, Alg.~\ref{alg:detailed-inference-ff-ebm}). Then, two nudged phases are applied to the last block and parameter gradients subsequently computed, as done classically (\texttt{BlockGradient} subroutine for the last block, Alg.~\ref{alg:block-grad-last-block}), with an extra computation to compute the error current to be applied to the penultimate block ($\delta s^{N-2}$). Then, the same procedure is recursed backward through blocks (Alg.~\ref{alg:block-grad}), until reaching first block.

\begin{algorithm}[ht!]{
\emph{Input:} $T$, $s^{N-2}_\star$, $\theta^{N-1}$, $\omega^{N-1}$, $\omega^N$, $\beta$, $\ell$, $y$\\
\emph{Output:$\delta s^{N-2}$}
}
    \caption{\texttt{BlockGradient} (for \bfseries{last} block)}
    \label{alg:block-grad-last-block}
    \begin{algorithmic}[1]
    \State $s^{N-1}_\beta \gets \texttt{Asynchronous}\left(T, \theta^{N-1}, \omega^{N-1}, \omega^N,  \beta, \ell , y\right)$ \Comment{Alg.~\ref{alg:asynchronous-last-block}}
    \State $s^{N-1}_{-\beta} \gets \texttt{Asynchronous}\left(T, \theta^{N-1}, \omega^{N-1}, \omega^N,  -\beta, \ell , y\right)$
    \State $g_{\omega^{N}} \gets \frac{1}{2}\left(\nabla_{s^{N-1}}\ell(F^{N}\left(s^{N-1}_\beta, \omega^N\right)) +  \nabla_{s^{N-1}}\ell(F^{N}\left(s^{N-1}_{-\beta}, \omega^N\right))\right)$
    \State $ g_{\theta^{N-1}} \gets \frac{1}{2\beta}\left(\nabla_{2} \widetilde{E}^{N-1}(s_\beta^{N-1}, \theta^{N-1}, s^{N-2}_\star, \omega^{N-1}) - \nabla_{2} \widetilde{E}^{N-1}(s^{N-1}_{-\beta}, \theta^{N-1}, s^{N-2}_\star, \omega^{N-2})\right)$
    \State $g_{\omega^{N-1}} \gets \frac{1}{2\beta}\left(\nabla_{4} \widetilde{E}^{N-1}(s_\beta^{N-1}, \theta^{N-1}, s^{N-2}_\star, \omega^{N-1}) - \nabla_{4} \widetilde{E}^{N-1}(s^{N-1}_{-\beta}, \theta^{N-1}, s^{N-2}_\star, \omega^{N-2})\right)$ 
    \State $ \delta s^{N-2} \gets \frac{1}{2\beta}\left(\nabla_{3} \widetilde{E}^{N-1}(s_\beta^{N-1}, \theta^{N-1}, s^{N-2}_\star, \omega^{N-1}) - \nabla_{3} \widetilde{E}^{N-1}(s^{N-1}_{-\beta}, \theta^{N-1}, s^{N-2}_\star, \omega^{N-2})\right)$
    \end{algorithmic}
\end{algorithm}

\begin{algorithm}[ht!]{
\emph{Input:} $T$, $s^{k-1}_\star$, $\theta^{k}$, $\omega^{k}$, $\beta$, $\delta s$\\
\emph{Output:$\delta s^{k-1}$}
}
    \caption{\texttt{BlockGradient} (for all blocks {\bfseries until penultimate})}
    \label{alg:block-grad}
    \begin{algorithmic}[1]
    \State $s^{k}_\beta \gets \texttt{Asynchronous}\left(T, \theta^{k}, \omega^{k}, \beta, \delta s\right)$ \Comment{Alg.~\ref{alg:asynchronous}}
    \State $s^{k}_{-\beta} \gets \texttt{Asynchronous}\left(T, \theta^{k}, \omega^{k}, -\beta, \delta s\right)$
    \State $ g_{\theta^{k}} \gets \frac{1}{2\beta}\left(\nabla_{2} \widetilde{E}^{k}(s_\beta^{k}, \theta^{k}, s^{k-1}_\star, \omega^{k}) - \nabla_{2} \widetilde{E}^{k}(s^{k}_{-\beta}, \theta^{k}, s^{k-1}_\star, \omega^{k})\right)$
    \State $ g_{\omega^{k}} \gets \frac{1}{2\beta}\left(\nabla_{4} \widetilde{E}^{k}(s_\beta^{k}, \theta^{k}, s^{k-1}_\star, \omega^{k}) - \nabla_{4} \widetilde{E}^{k}(s^{k}_{-\beta}, \theta^{k}, s^{k-1}_\star, \omega^{k})\right)$
    \State $ \delta s^{k-1} \gets \frac{1}{2\beta}\left(\nabla_{3} \widetilde{E}^{k}(s_\beta^{k}, \theta^{k}, s^{k-1}_\star, \omega^{k}) - \nabla_{3} \widetilde{E}^{k}(s^{k}_{-\beta}, \theta^{k}, s^{k-1}_\star, \omega^{k})\right)$
    \end{algorithmic}
\end{algorithm}

\begin{algorithm}[ht!]
    \caption{Detailed implicit BP-EP gradient chaining}
    \label{alg:implicit-bp-ep-chain-detailed}
    \begin{algorithmic}[1]
    \State $s^1_\star, \cdots, s^{N-1}_\star \gets \texttt{Forward}\left(T_{\rm free}, x, W\right)$ \Comment{Alg.~\ref{alg:detailed-inference-ff-ebm}}
    \State $\delta s \gets \texttt{BlockGradient}\left(T_{\rm nudge}, s^{N-2}_\star, \theta^{N-1}, \omega^{N-1}, \omega^{N}, \beta, \ell, y \right)$ \Comment{Alg.~\ref{alg:block-grad-last-block}}
    \For{$k=N - 2 \cdots 1$}
        \State $\delta s \gets \texttt{BlockGradient}\left(T_{\rm nudge}, s^{k-1}_\star, \theta^k, \omega^k, \beta, \delta s\right)$ \Comment{Alg.~\ref{alg:block-grad}}
    \EndFor
    \end{algorithmic}
\end{algorithm}

\paragraph{Details about the implicit differentiation baseline.}
We describe our implementation of Implicit Differentiation (ID) inside Alg.~\ref{alg:id}. First, we relax all blocks sequentially to equilibrium following Alg.~\ref{alg:detailed-inference-ff-ebm} and we do not track gradients throughout this first phase, using $T_{\rm free}$ fixed-point iteration steps per block. \emph{Then}, initializing the block states with those computed at the previous step, we re-execute the same procedure (still with Alg.~\ref{alg:detailed-inference-ff-ebm}), this time \emph{tracking gradients} and using $T_{\rm nudge}$ steps fixed-point iteration steps for each block. Then, we use automatic differentiation to backpropagate through the last $T_{\rm nudge}$ steps for each block, namely backpropagating, backward in time, \emph{through equilibrium}. 

\begin{algorithm}[ht!]
    \caption{Our implementation of ID}
    \label{alg:id}
    \begin{algorithmic}[1]
    \State Without tracking gradients:\Comment{e.g. \texttt{with torch.no\_grad()}}
    \State \quad $s^1_\star, \cdots, s^{N-1}_\star \gets \texttt{Forward}\left( T_{\rm free}, x, W \right)$ \Comment{Alg.~\ref{alg:detailed-inference-ff-ebm}}
    \State Initialize block states at $s^1_\star, \cdots, s^{N-1}_\star$
    \State $\hat{o}_\star \gets \texttt{Forward}\left( T_{\rm nudge}, x, W \right)$  \Comment{This time gradients are tracked}
    \State  $\mathcal{C} \gets \ell(\hat{o}_\star, y)$
    \State  Backpropagate $\mathcal{C}$ backward through the last $T_{\rm nudge}$ steps for each block\Comment{e.g. \texttt{C.backward()}}
    \end{algorithmic}
\end{algorithm}

\newpage

\subsection{Details about the static gradient analysis}
\paragraph{Single block case.}  Let us consider a single block with parameters $\theta$ and state $s$. Earlier, we showed that the (free) dynamics of a block read in the general form as:

\begin{equation}
  s_{t + 1} \gets \nabla_1 \Phi\left(s_{t}, \theta, x\right),
  \label{eq:recurrent-dyn}
\end{equation}

where we have ignored the activation function, as done by \cite{ernoult2019updates}. Therefore, the computational graph for Eq.~(\ref{eq:recurrent-dyn}) is that of a recurrent neural network with a static input ($x$) where the parameters $\theta$ are shared \emph{across time steps}. Therefore, one could view Eq.~(\ref{eq:recurrent-dyn}) as:

\begin{equation}
    s_{t + 1} \gets \nabla_1 \Phi\left(s_{t}, \theta(t) = \theta, x \right).
  \label{eq:recurrent-dyn-2}
\end{equation}

At the end of the recurrent dynamics Eq.~(\ref{eq:recurrent-dyn-2}), a loss is computed based on the output of the last block, namely:

\begin{equation}
    \mathcal{C} = \ell\left(s_T, y\right),
\end{equation}

where $s_T$ is the state of the last block after $T$ iterations of the recurrent dynamics Eq.~(\ref{eq:recurrent-dyn-2}). In principle, the gradient of the loss $\mathcal{C}$ with respect to $\theta$ reads as:

\begin{equation}
    g_{\theta} = d_{\theta(1)} \mathcal{C} + d_{\theta(2)} \mathcal{C} + \cdots + d_{\theta(T)} \mathcal{C}.
\end{equation}

In other words,  the gradient of the loss $\mathcal{C}$ with respect to $\theta$ is the sum of all ``sensitivities'' of the loss to the same parameter taken at all timesteps where it is involved. One can define the \emph{truncated} gradient at $T-t$, i.e. \emph{$t$ steps backward in time from $T$} as:

\begin{equation}
    \hat{g}^{\rm AD}_\theta(t) := \sum_{k=0}^{t-1} d_{\theta(T - k)} \mathcal{C},
\end{equation}

such that $\hat{g}^{\rm AD}_\theta(T) = g_{\theta}$. AD stands for ``automatic differentiation'' because this is how these gradients may be computed in practice -- it can also be viewed as an instantiation of \emph{backpropagation through time} (BPTT) with a static input. Then, writing the nudged dynamics prescribed by EP as:

\begin{equation}
    s_{t + 1}^{\pm\beta} := \nabla_1 \Phi\left(s^{\pm\beta}_{t}, \theta, x \right) \mp \beta \nabla_1 \ell(s_t^\beta, y),
  \label{eq:recurrent-dyn-nudged}
\end{equation}

we define the truncated EP gradient estimate at timestep t \emph{through the second phase} as:

\begin{equation}
    \hat{g}^{\rm EP}_\theta(\beta, t) := \frac{1}{2\beta}\left(\nabla_2 \Phi\left(s^\beta_t, \theta, x\right) -  \nabla_2 \Phi\left(s^{-\beta}_t, \theta, x\right)\right).
\end{equation}

Furthermore, we assume that $s$ is already at its steady state $s_\star$ over the last $K$ steps, i.e. $T-K$, $T-K-1$, ..., $T$, so that AD here occurs \emph{through equilibrium}, then the following equality holds \citep{ernoult2019updates}:

\begin{equation}
  \forall t = 1 \cdots K: \quad \lim_{\beta \to 0}\hat{g}^{\rm EP}_\theta(\beta, t) = \hat{g}^{\rm AD}_\theta(t)
  \label{eq:gdd}
\end{equation}

This property was called the \emph{Gradient Descending Updates} (GDU) property. In practice, the GDU property is a tremendously helpful sanity check to debug code to make sure EP gradients are correctly computed. Our coding work heavily relied on this property.

\paragraph{Multiple blocks.} In ff-EBMs, the GDU property may immediately hold for the last block. For penultimate block backwards, the sole difference compared to the previous paragraph is simply the nudging force that is applied to the model, for which Eq.~(\ref{eq:recurrent-dyn-nudged}) may simply be changed into:

\begin{equation}
    s_{t + 1}^{\pm\beta} \gets \nabla_1 \Phi\left(s^{\pm\beta}_{t}, \theta, x \right) \mp \beta \delta s,
  \label{eq:recurrent-dyn-nudged-pen}
\end{equation}

where $\delta s$ is a \emph{constant} error signal. Therefore, so long as the error signal $\delta s$ computed by EP matches that computed by AD (which we proved inside Theorem~\ref{theorem:main-result-formal}), the GDU property should still hold.




\newpage
\subsection{Experimental Details}
\subsubsection{Datasets}

Simulations were run on CIFAR-10, CIFAR-100 and Imagenet32 datasets, all consisting of color images of size  32 $\times$ 32 pixels. CIFAR-10  \cite{krizhevsky2009learning} includes 60,000 color images of objects and animals . Images are split into 10 classes, with 6,000 images per class. Training data and test data include 50,000 images, and 10,000 images respectively. Cifar-100 \cite{krizhevsky2009learning} likewise comprises 60,000, and features a diverse set of objects and animals split into 100 distinct classes. Each class contains 600 images. Like CIFAR-10, the dataset is divided into a training set with 50,000 images and a test set containing the remaining 10,000 images. The ImageNet32 dataset \cite{chrabaszcz2017downsampled} is a downsampled version of the original ImageNet dataset \cite{russakovsky2015imagenet} containing 1,000 classes with 1,281,167 training images, 50,000 validation images,  100,000 test images and 1000 classes.

\subsubsection{Data Pre-processing}
All data were normalized according to statistics shown in \ref{tab:data_normalization}and augmented with 50 \% random horizontal flips. Images were randomly cropped and padded with the last value along the edge of the image. 

\begin{table}[ht]
\centering
\caption{Data Normalization. Input images were normalized by conventional mean ($\mu$) and standard deviation ($\sigma$) values for each dataset. All images used are color (three channels).}
\label{tab:data_normalization}
\begin{tabular}{@{}lll@{}}
\toprule
Dataset & Mean ($\mu$) & Std ($\sigma$) \\
\midrule
CIFAR-10/100 & (0.4914, 0.4822, 0.4465) & (0.2470, 0.2435, 0.2616) \\
Imagenet32 & (0.485, 0.456, 0.406) & (0.3435, 0.336, 0.3375) \\
\bottomrule
\end{tabular}
\end{table}

\subsubsection{Simulation Details}
\paragraph{Weight Initialization}
Equilibrium propagation, similar to other machine learning paradigms reliant on fixed-point iteration \citep{bai2019deep}, is highly sensitive to initialization statistics \citep{agarwala2022deep}, hence conventionally difficult to tune, and requiring many iterations for the three relaxation phases. Initialization of weights as Gaussian Orthogonal Ensembles (GOE) ensures better stability  (reduced variance) and, combined with other stabilizing measures discussed below, empirically yields faster convergence. 
\\
According to GOE, weights are intialized as:\\

\[ W_{ij} \sim \begin{cases}
\mathcal{N}(0, \frac{V}{N}), & \text{if } i \neq j \\
\mathcal{N}(0, \frac{2V}{N}), & \text{if } i = j
\end{cases} \]

where \(\mathcal{N}(\mu, \sigma^2)\) denotes a Gaussian (normal) distribution with mean \(\mu\) and variance \(\sigma^2\).

\paragraph{State Initialization}
All states are initialized as zero matrices.
\paragraph{Activation Functions}
An important detail for faithful reproduction of these experiments is the choice and placement of activation functions applied during the iterative root-finding procedure. In the literature, clamping is conventionally applied at each layer, with the exception of the final layer, where it is sometimes included e.g. \cite{scellier2024energy}, and sometimes omitted \cite{laborieux2021scaling}. For these experiments we use both the standard hard activation employed in \cite{ernoult2019updates} and \citep{scellier2024energy}, and the more conservative one given in \citep{laborieux2021scaling}. Details are given in \ref{tab:activation}. 

\begin{table}[ht]
\centering
\caption{Activation functions}
\label{tab:activation}
\begin{tabular}{@{}lll@{}}
\toprule
Name & Description  & Source \\
\midrule
ernoult & $\sigma({x}) = \max(\min({x}, 1), 0)$  & \citep{ernoult2019updates} \\
laborieux & $\sigma({x}) = \max(\min(0.5 \times {x}, 1), 0)$  & \citep{laborieux2021scaling} \\
\bottomrule
\end{tabular}
\end{table}
In practice, we find that the laborieux activation, in conjunction with GOE, and the omission of clamping at the output of \emph{each block} significantly enhances gradient stability and speeds convergence in the multiple block setting. In the interest of multi-scale uniformity and consistency with previous literature \citep{laborieux2021scaling}  \cite{ernoult2019updates}, however, we apply clamping using the ernoult activation on \emph{all layers} in our 6-layer architecture. \\\\ For the scaling experiments, we apply laborieux activation at every layer \emph{except} the output of each block.  For the 12-layer splitting experiment, we do the same, omitting clamping from the output of the final layer of \emph{each} block in the block-size$=$4 and block-size$=$3 experiments. However, in the block-size$=$2 case we clamp the output of the second and fourth blocks to preserve dynamics of the block-size$=$4 split. Such consistency is not possible for the block-size$=$3 experiment, constituting a possible discrepancy in the scaling dynamics.

\paragraph{Cross-entropy Loss and Softmax Readout}

Following \citep{laborieux2021scaling}, all models were implemented such that the output $y$ is removed from the system (e.g. not included in the relaxation dynamics) but is instead the function of a weight matrix: $W_{\text{out}}$ of size $\text{dim}(y) \times \text{dim}(s)$, where $s$ is the state of the final layer. For each time-step $t$, $ \hat{y}_t = \text{softmax}(W_{\text{out}}s_t). $

The cross-entropy cost function associated with the softmax readout is then:

$$ l(s, y, W_{\text{out}}) = -\sum_{c=1}^C y_c \log(\text{softmax}_c(W_{\text{out}} \cdot s)).  $$
\\ It is important to note that by convention we refer to architectures throughout this text to the exclusion of the softmax readout, which is technically an additional layer, though not involved in the relaxation process.

\paragraph{Architecture}
Details of architectures used in the experiments reported in \ref{table:split}  and \ref{tab:scaling} are given in table \ref{tab:architectures}. All convolutional layers used in experiments are of kernel size 3 and stride and padding 1. Max-pooling was applied with a window of $2 \times 2$ and stride of 2. For the 6-layer model used in the \ref{table:split} batch norm was applied \emph{after} the first layer convolution and pooling operation. All other models in both experiments use batch-normalization on the first layer of each block \emph{after} convolution and pooling (where applied). These details exclude the Linear Softmax readout of size ${n}$ classes.

\begin{table}[ht!]
\centering
\caption{Architectures for training experiments.}
\label{tab:architectures}
\footnotesize
\setlength{\tabcolsep}{4pt}

\begin{tabular*}{\textwidth}{@{\extracolsep{\fill}}llll} 
   \toprule
   Hyperparameters & \multicolumn{3}{c}{\hspace*{-1.5cm} Architecture} \\ 
                   & 6-layer  & 12-layer & 15-layer \\ \midrule

Kernel size & 3 & 1 & 1  \\
Stride & 1  & 1 & 1\\
Padding & 1 & 1 & 1 \\
SoftPool stride & 2 & 2 & 2 \\  
SoftPool window & 2 $\times$ 2 & 2 $\times$ 2 & 2 $\times$ 2 \\
  Channel sizes & 
  \begin{tabular}[t]{@{}l@{}}      
     128 \\
      256 $\times$ 2 \\
      512 $\times$ 2 \\
      256 (Linear)
    \end{tabular} & 
  \begin{tabular}[t]{@{}l@{}}     
      128 $\times$ 4 \\
      256 $\times$ 4 \\
      512 $\times$ 4
    \end{tabular} & 
  \begin{tabular}[t]{@{}l@{}}     
      64 $\times$ 2 \\
      128 $\times$ 4 \\
      256 $\times$ 4\\
      512 $\times$ 5\\
    \end{tabular} \\
Pooling Applied & [1-1-1-1-0-0]& [0-0-1-0-1-0-0-0-1-0-0-0-]&[0-0-1-0-0-0-1-0-0-1-0-0-0-0] \\
\hline 
\\ \bottomrule

\end{tabular*}
\end{table}

\paragraph{Hyperparameters}
Hyperparameters and implementation details for \ref{tab:scaling} and \ref{table:split} are given in table \ref{tab:hyperparameters}. Note that all architectural details for the 12-layer models are identical across splitting and scaling experiments. Additionally, identical hyperparameters were used for Cifar-100 and Imagenet experiments of \ref{tab:scaling}. Unlike previous literature, the use of GOE intialization eliminates the need for separate layerwise learning rates and initialization parameters. One noteworthy detail is that only 100 epochs were used for the larger model for \ref{tab:scaling} compared with 200 epochs for the smaller 12-layer model. This was due to prohibitively long run-time of training the larger model. Noteably performance still significantly improves with decreased overall runtime.

\begin{table}[ht!]
\centering
\footnotesize  
\setlength{\tabcolsep}{4pt} 
\caption{Hyperparameters used for training experiments.}
\label{tab:hyperparameters}
\begin{tabular*}{\textwidth}{@{\extracolsep{\fill}}lllll}
    \toprule
    Experiment & \multicolumn{2}{c}{Splitting(Cifar-10)} & \multicolumn{2}{c}{Scaling(CIFAR-100/Imagenet)} \\
                    &  6-Layer   & 12-Layer   & 12-Layer   & 15-Layer   \\ \midrule
    Batch size      & \multicolumn{2}{c}{128}     & \multicolumn{2}{c}{256}     \\
    Activation      & \multicolumn{2}{c}{ernoult}     & \multicolumn{2}{c}{laborieux}    \\
    
    Epochs          & 200                  & 200                     & 200                  & 100                     \\
    $\beta$         & 0.2                 & 0.2                   & 0.2                 & 0.2                   \\
    $T_{\text{free}}$    & 60                  & 20                   & 15                  & 15                   \\
    $T_{\text{nudge}}$   & 20                  & 5                    & 5                  & 5    \\
    
    V *   &   $8.4e^{-4}$   & $5.9e^{-5}$ & $4.9e^{-5}$& $1e^{-5}$\\

    Initial LRs**    & $1e^{-4}$   & $5e^{-5}$   & $5e^{-5}$     & $2e^{-5}$   \\
    Final LRs    &  $ 1e^{-6}$ & $ 1e^{-6}$   & $1e^{-7}$     & $1e^{-7}$    \\
    Weight decay    & $3e^{-4}$&$3e^{-4}$&$3e^{-4}$&$3e^{-4}$            \\
                \\ \bottomrule
\end{tabular*}
\footnotesize{* All models used Gaussian Orthogonal weight initialization scheme parameterized by V\\** Learning rates were decayed with cosine annealing without restart [S7].}
\end{table}
\paragraph{Other details}
All experiments were run using Adam optimizer \citep{kingma2014adam}and Cosine Annealing scheduler\citep{loshchilov2016sgdr} with minimum learning rates indicated in \ref{tab:hyperparameters} and maximum T equal to epochs (i.e. no warm restarts). Code was implemented in Pytorch 2.0 and all simulations were run on A-100 SXM4 40GB GPUs.
\end{appendices}

\newpage

\newpage
\clearpage
\section*{NeurIPS Paper Checklist}

\begin{enumerate}

\item {\bf Claims}
    \item[] Question: Do the main claims made in the abstract and introduction accurately reflect the paper's contributions and scope?
    \item[] Answer: \answerYes{}

\item {\bf Limitations}
    \item[] Question: Does the paper discuss the limitations of the work performed by the authors?
    \item[] Answer: \answerYes{} 
    \item[] Justification: we have a dedicated paragraph in the ``Discussion'' section of the paper which explicitly mentions limitations and future work.

\item {\bf Theory Assumptions and Proofs}
    \item[] Question: For each theoretical result, does the paper provide the full set of assumptions and a complete (and correct) proof?
    \item[] Answer: \answerYes{}
    \item[] Justification: All proofs are included in the appendix.

    \item {\bf Experimental Result Reproducibility}
    \item[] Question: Does the paper fully disclose all the information needed to reproduce the main experimental results of the paper to the extent that it affects the main claims and/or conclusions of the paper (regardless of whether the code and data are provided or not)?
    \item[] Answer: \answerYes{}
     \item[] Justification: detailed information to reproduce experiment is provided in the appendix.

\item {\bf Open access to data and code}
    \item[] Question: Does the paper provide open access to the data and code, with sufficient instructions to faithfully reproduce the main experimental results, as described in supplemental material?
    \item[] Answer: \answerNo{} 
    \item[] Justification: We only employed public datasets and will release our code at the end of the reviewing process. In the meanwhile, we provided very detailed pseudo-algorithms in the appendix.

\item {\bf Experimental Setting/Details}
    \item[] Question: Does the paper specify all the training and test details (e.g., data splits, hyperparameters, how they were chosen, type of optimizer, etc.) necessary to understand the results?
    \item[] Answer: \answerYes{}
     \item[] Justification: all these details are provided in the appendix. 

\item {\bf Experiment Statistical Significance}
    \item[] Question: Does the paper report error bars suitably and correctly defined or other appropriate information about the statistical significance of the experiments?
    \item[] Answer: \answerYes{} 
    \item[] Justification: we perform each of our training simulations on 3 different seeds and reported mean and standard deviation of the resulting performance. 

\item {\bf Experiments Compute Resources}
    \item[] Question: For each experiment, does the paper provide sufficient information on the computer resources (type of compute workers, memory, time of execution) needed to reproduce the experiments?
    \item[] Answer: \answerYes{} 
    \item[] Justification: this information is also inside our appendix.
    
\item {\bf Code Of Ethics}
    \item[] Question: Does the research conducted in the paper conform, in every respect, with the NeurIPS Code of Ethics \url{https://neurips.cc/public/EthicsGuidelines}?
    \item[] Answer: \answerYes{} 

\item {\bf Broader Impacts}
    \item[] Question: Does the paper discuss both potential positive societal impacts and negative societal impacts of the work performed?
    \item[] Answer: \answerYes{} 
    \item[] Justification: we wrote a dedicated paragraph inside our ``Discussion'' section.

\item {\bf Safeguards}
    \item[] Question: Does the paper describe safeguards that have been put in place for responsible release of data or models that have a high risk for misuse (e.g., pretrained language models, image generators, or scraped datasets)?
    \item[] Answer: \answerNA{} 
    \item[] Justification: Our work poses no such risk at present as it only provides proof-of-concepts for systems which do not yet exist.

\item {\bf Licenses for existing assets}
    \item[] Question: Are the creators or original owners of assets (e.g., code, data, models), used in the paper, properly credited and are the license and terms of use explicitly mentioned and properly respected?
    \item[] Answer: \answerNA{} 
    \item[] Justification: this work does not use existing assets.

\item {\bf New Assets}
    \item[] Question: Are new assets introduced in the paper well documented and is the documentation provided alongside the assets?
    \item[] Answer: \answerNA{} 
 \item[] Justification: this paper does not release new assets.

\item {\bf Crowdsourcing and Research with Human Subjects}
    \item[] Question: For crowdsourcing experiments and research with human subjects, does the paper include the full text of instructions given to participants and screenshots, if applicable, as well as details about compensation (if any)? 
    \item[] Answer: \answerNA{} 
    \item[] Justification: our paper does not involve crowdsourcing nor research with human subjects.

\item {\bf Institutional Review Board (IRB) Approvals or Equivalent for Research with Human Subjects}
    \item[] Question: Does the paper describe potential risks incurred by study participants, whether such risks were disclosed to the subjects, and whether Institutional Review Board (IRB) approvals (or an equivalent approval/review based on the requirements of your country or institution) were obtained?
    \item[] Answer: \answerNA{} 
    \item[] Justification: our paper does not involve crowdsourcing nor research with human subjects.

\end{enumerate}

\end{document}